\documentclass[11pt]{article}
\usepackage[left=1in,top=1in,right=1in,bottom=1in,letterpaper]{geometry}
\usepackage{microtype}
\usepackage{graphicx}
\usepackage{subfigure}
\usepackage{booktabs} 

\usepackage[utf8]{inputenc} 
\usepackage[T1]{fontenc}    
\usepackage{hyperref}

\usepackage{url}            
\usepackage{booktabs}       
\usepackage{amsfonts}       
\usepackage{nicefrac}       
\usepackage{microtype}      
\usepackage{xcolor}         
\usepackage{color}
\usepackage{bm,multirow,xspace}
\usepackage{enumerate,mathtools}
\usepackage[inline]{enumitem}
\usepackage{graphicx}
\usepackage{epstopdf,amsmath}
\usepackage{bbm}

\usepackage{authblk}

\usepackage{ amssymb }
\usepackage{mathrsfs}
\usepackage{makecell}

\allowdisplaybreaks

\usepackage{amsmath}    
\usepackage{amsmath,color}
\usepackage{algorithm}
\usepackage{algorithmic}

\usepackage{graphicx}
\usepackage{epstopdf}
\usepackage{caption}

\usepackage{bbm}
\allowdisplaybreaks
\newtheorem{definition}{Definition}

\DeclareMathOperator*{\argmin}{argmin}
\DeclareMathOperator{\sign}{sign}
\newcommand{\abs}[1]{\left\lvert #1 \right\rvert}
\DeclareMathOperator*{\iid}{\texttt{iid}}
\newcommand{\norm}[1]{\left\lVert #1 \right\rVert}

\newcommand{\expec}[2]{\mathbb{E}_{#2}\left[ #1 \right] }
\newcommand\numberthis{\addtocounter{equation}{1}\tag{\theequation}}  

\def\fn[#1]#2{{f_{#1}\left(x_{#2}\right)}}

\newtheorem{lemma}{Lemma}[section]
\newtheorem{theorem}{Theorem}[section]

\newtheorem{proposition}{Proposition}[section]
\newtheorem{assumption}{Assumption}[section]
\newtheorem{remark}{Remark}

\def\phidm{{\phi_{DM}}}
\def\phidi{{\phi_{DI}}}

\def\exp{{\rm exp}}

\def\bt{{\bar{\theta}}}
\def\tb{\theta^b}

\newenvironment{talign*}
 {\csname align*\endcsname}
 {\endalign}

\author{Abhishek Roy\thanks{Halıcıoğlu Data Science Institute, University of California, San Diego. \texttt{a2roy@ucsd.edu}.} 
\and Prasant Mohapatra\thanks{Department of Computer Science, University of California, Davis. \texttt{pmohapatra@ucdavis.edu}. } 
}

\title{Fairness Uncertainty Quantification: How certain are you that the model is fair?}

\begin{document}

\maketitle
\begin{abstract}
Fairness-aware machine learning has garnered significant attention in recent years because of extensive use of machine learning in sensitive applications like judiciary systems. Various heuristics, and optimization frameworks have been proposed to enforce fairness in classification \cite{del2020review} where the later approaches either provides empirical results or provides fairness guarantee for the exact minimizer of the objective function \cite{celis2019classification}. In modern machine learning, Stochastic Gradient Descent (SGD) type algorithms are almost always used as training algorithms implying that the learned model, and consequently, its fairness properties are random. Hence, especially for crucial applications, it is imperative to construct Confidence Interval (CI) for the fairness of the learned model. In this work we provide CI for test unfairness when a group-fairness-aware, specifically, Disparate Impact (DI), and Disparate Mistreatment (DM) aware linear binary classifier is trained using online SGD-type algorithms. We show that asymptotically a Central Limit Theorem holds for the estimated model parameter of both DI and DM-aware models. We provide online multiplier bootstrap method to estimate the asymptotic covariance to construct online CI. To do so, we extend the known theoretical guarantees shown on the consistency of the online bootstrap method for unconstrained SGD to constrained optimization which could be of independent interest. We illustrate our results on synthetic and real datasets.
\end{abstract}
\vspace{-0.25in}
\section{Introduction}
\vspace{-0.05in}
Machine learning has become pervasive across different fields over the past decade. Widespread use of machine learning in crucial applications like medical and judiciary system has raised serious concerns regarding fairness, privacy, and interpretability of the learned models. Various heuristic methods and optimization frameworks have been proposed to enforce fairness in classification \cite{del2020review}. The later line of work either provides empirical results or provides fairness guarantee for the exact minimizer of the objective function \cite{celis2019classification} but does not shed any light on the variations of the fairness of the learned model resulting from the dynamics of the optimization algorithm used for training. Online stochastic optimization algorithms like Stochastic Gradient Descent (SGD) have become the training algorithm of choice given the success these algorithms have enjoyed for huge datasets. Stochasticity of training dynamics imply that the learned model is also random. So an empirical, even a theoretical bound on expected fairness may be inadequate for some high-risk applications, e.g., medical applications \cite{begoli2019need}, and judiciary system where the sentence of an individual is decided by a machine learning model \cite{zafar2019fairness}.
\vspace{-0.05in}
\begin{figure}[h] 
    \centering
    \includegraphics[width=82mm,height=56mm]{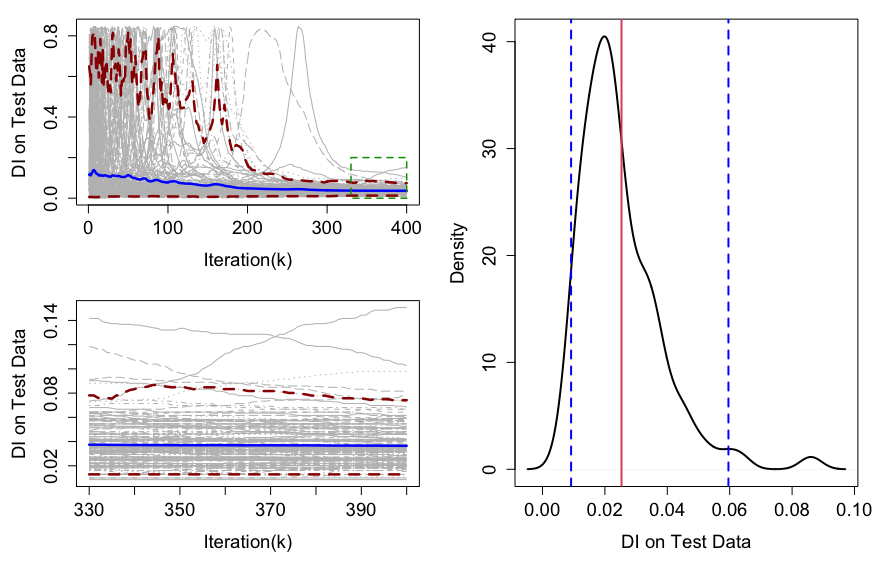}
    \caption{DI of the trained model varies considerably over repetitions, even after the mean level has stabilized. The mean DI after training is $0.025$ whereas the $97.5\%$ quantile is $0.042$ which is $64\%$ more!}
\end{figure}
\vspace{-0.1in}
To see this, consider the synthetic dataset presented in \cite{zafar2019fairness} vulnerable to Disparate Impact (DI), i.e., a trained classifier which is unaware of the unfairness, has different probability of predicting positive class for the two groups of sensitive attribute. When the classifier is trained with an online SGD-type algorithm operating on $\iid$ data stream, subject to the fairness constraints as discussed in \cite{zafar2019fairness}, the mean value of DI becomes significantly small~0.025 on test data. But DI of the learned model varies significantly over $200$ repetitions even after keeping the test data and the initialization of the optimization algorithm fixed (see Figure~\ref{fig:motivation}). This is due to the stochastic nature of the optimization algorithm. Motivated by this, we study the following problem in this paper. 
\vspace{-0.05in}
\begin{quote}
\textit{How to quantify the uncertainty present in a fairness-aware linear classifier when trained with an online stochastic optimization algorithm?}
\end{quote}
\vspace{-0.05in}
The history of Uncertainty Quantification (UQ), known as inference in classical statistics, go way back. There has been considerable research in recent years on UQ of machine learning, especially, neural networks \cite{wu2021quantifying,psaros2023uncertainty,caldeira2020deeply,zhang2022explainable,zhan2022uncertainty}. Various methods have been proposed for UQ of neural network parameters under different settings \cite{psaros2023uncertainty}, e.g., Bayesian method, method of Functional Priors and ensemble method. Unlike us, Bayesian method, and method of functional priors do not take optimization into account. The method of ensembles require repeated training whereas we characterize the asymptotic distribution explicitly and to estimate the asymptotic covariance we use only one trajectory of training data. 

\textbf{Problem Setup} We begin by emphasizing that we are not designing a fairness enforcing framework; rather we characterize the uncertainty of such a framework when trained with stochastic algorithm. Specifically, we construct CI for DI and Disparate Mistreatment (DM) \cite{zafar2019fairness} of a linear classifier trained with an online stochastic optimization algorithm.  For DI-aware model we use the constrained optimization framework introduced in \cite{zafar2019fairness}; see Section~\ref{sec:di} for details. As training algorithm, we choose Stochastic Dual-Averaging (SDA) to avoid the requirement of minibatches to estimate the gradient of the objective function. For the DM-aware model, we consider a penalized version of the problem introduced in \cite{zafar2019fairness} in order to transform the optimization with convex-concave constraint set to a strongly-convex problem. This, instead of using heuristic method to solve optimization with convex-concave constraints, enables us to provide theoretical results on the uncertainty. We use vanilla SGD without minibatch to solve this penalized problem.  

\textbf{Our Contributions} First, we show that a Central Limit Theorem (CLT) holds for the model parameter of DI and DM-aware models when trained with SDA and SGD respectively for strongly-convex loss functions. Historically, Cross-Entropy (CE) loss has been the loss function of choice for classification. But recently, it has been observed that squared loss is comparable, and even better than CE loss in terms of classification test accuracy \cite{demirkaya2020exploring,hui2020evaluation}. Here we shed lights on the fairness of the models trained with squared loss and CE loss.

We show that the covariance of the asymptotic distribution depends on the global optima of the fairness-aware optimization which is unknown beforehand. So the asymptotic covariance needs to be estimated in an online fashion to construct CIs for DI and DM. To do so in an online manner, we propose an online bootstrap algorithm. We theoretically show that the proposed bootstrap estimator converges in distribution to the asymptotic distribution of the model parameter. This result may be of independent interest for covariance estimation of constrained stochastic optimization literature. Moreover, one can construct CI of the model parameter from these bootstrap samples to identify the significant features for fairness-aware classification leading to better \textit{interpretability}. To the best of our knowledge, this is the first work on online UQ for group fairness measures DI and DM through the lens of optimization. 

\textbf{The summary of novelties of this work is given below.}
\vspace{-0.1 in}
\begin{enumerate}[label=(\alph*)] 
    \item We show that the estimated model parameter of a fairness-aware linear classifier, trained with SGD-type algorithm under strongly-convex loss, follows a CLT (Theorem~\ref{th:duchiclt}, and Theorem~\ref{th:dmclt}).
    \item We propose an online bootstrap algorithm to construct asymptotically valid CI for DI and DM of a trained model (Algorithm~\ref{alg:boot}).
    \item We theoretically show that the online bootstrap estimator converges in distribution to the asymptotic distribution of the model parameter (Theorem~\ref{th:bootdistconv}). This result is novel in the context of online covariance estimation for constrained stochastic optimization as well. 
\end{enumerate}
\vspace{-0.1in}
In Section~\ref{eq:relwork} we present related literature. In Section~\ref{sec:notprob}, we introduce the notations and the problem statement. We present the results on asymptotic distribution of the model parameters in Section~\ref{sec:asymp}. We show our results on online bootstrap-based covariance estimation in Section~\ref{sec:onlinecov}. We illustrate the validity of our theoretical results in Section~\ref{sec:results} on real and synthetic dataset. Section~\ref{sec:disc} concludes the paper.
\vspace{-0.25in}
\section{Related Work}\label{eq:relwork}
\textbf{UQ in machine learning:} Recently UQ in machine learning has gained significant attention; see \cite{abdar2021review,psaros2023uncertainty} for a few and by no means exhaustive list of significant works on this topic. 

\textbf{UQ in optimization} Asymptotic properties of stochastic approximation algorithms have been studied in depth in the phenomenons works \cite{kushner2003stochastic,benveniste2012adaptive}. In the seminal paper \cite{polyak1992acceleration} it is shown that the averaged iterate of SGD follows a CLT and achieves optimal asymptotic variance. \cite{chen2016statistical, zhu2021online}, and \cite{fang2018online} propose online batch-means estimator and online bootstrap-based approaches respectively for online estimation of the asymptotic covariance for SGD. 

\textbf{UQ in fairness} 
The literature gets really scarce here. To the best of our knowledge, the only work studying similar problem is \cite{maity2021statistical}. \cite{maity2021statistical} develops asymptotically valid CI for individual unfairness. They use gradient-flow based adversarial attack to generate a test statistic to measure individual unfairness and show asymptotic normality of this test statistic. Our work differs from \cite{maity2021statistical} in the following ways:
\begin{enumerate*}
\item we focus on group fairness measures like DI and DM instead of individual fairness,
\item we capture the uncertainty resulting from the use of stochastic training algorithm as opposed to uncertainty under adversarial attack on features,
\item Note that the asymptotic validity of the CIs provided in \cite{maity2021statistical} depends on the fact that the test statistic is sum of function of $\iid$ variables. Whereas in this work the SGD-type algorithms inherently induces a Markov chain, even with $\iid$ data. Online CI estimation from a single-trajectory of a Markov chain requires a more involved analysis. Unlike \cite{maity2021statistical}, sample-covariance cannot be used as the covariance estimator here (see \cite{zhu2021online,fang2018online} for details). 
\end{enumerate*} 
\vspace{-0.1in}
\section{Notations and Problem}\label{sec:notprob}
Let $h(\cdot;\theta):\mathbb{R}^d \to\{-1,1\}$ be a classifier parameterized by $\theta\in\mathbb{R}^{d_\theta}$. Given a sample $(x,z,y)$, where $x\in\mathbb{R}^d$ is the feature, $z\in\{0,1\}$ is the sensitive attribute, and $y\in\{-1,1\}$ is the label. We will use $\rho(w)$ to denote the density of a random variable $W$. Let $L(\cdot,W):\Theta\to\mathbb{R}^+\cup \{0\}$ be the loss function where $W\coloneqq(x,y)$, and $l(\theta)\coloneqq\expec{L(\theta,W)}{}$. Let $\delta_\theta(x)$ denote the decision boundary,i.e., $h(x;\theta)=\sign(\delta_{\theta}(x))$. As mentioned above, in this work we use DI and DM as measures of fairness. Let $W_2(X,Y)$ denote the Wasserstein-2 distance between random variables $X$ and $Y$. For some statements we use $DI[DM]$ as subscript to imply that the statement holds for DI and DM both.
\begin{definition}[Disparate Impact]
A binary classifier is said to be fair with respect to DI if the probability of the classifier predicting positive class does not change conditional on the value of the sensitive feature, i.e.,
\begin{align*}
    P(h(x;\theta)=1|z=0)= P(h(x;\theta)=1|z=1).
\end{align*}
We use the following quantity to measure DI,
\begin{align*}
    \phidi(\theta)=&\lvert P(h(x;\theta)=1|z=0)-P(h(x;\theta)=1|z=1)\rvert.\numberthis\label{eq:phididef}
\end{align*}
\end{definition}
\vspace{-0.1in}
A classifier does not suffer from DM if the misclassification rate for the two sensitive groups are same. DM can be defined in terms of False Positive Rate (FPR), False Negative Rate (FNR), False Omission Rate (FOR), and False Discovery Rate (FDR). Although our methods can be applied to any of these choices, in this work we show the results corresponding to DM with respect to FPR.
\vspace{-0.05in}
\begin{definition}[Disparate Mistreatment]
A binary classifier is said to be fair with respect to DM if the False Positive Rate of the classifier does not change conditional on the value of the sensitive feature, i.e.,
\begin{align*}
    P(h(x;\theta)=1|y=-1,z=0)= P(h(x;\theta)=1|y=-1,z=1).
\end{align*}
We use the following quantity to measure DM unfairness
\begin{align*}
    \phidm(\theta)
    =\left| P(h(x;\theta)=1|y=-1,z=0)-
    P(h(x;\theta)=1|y=-1,z=1)\right|. \numberthis\label{eq:phidmdef}
\end{align*}
\end{definition}
\vspace{-0.15in}
\section{Asymptotics of Fairness-aware Classification}\label{sec:asymp}
In this section we show that when a fairness-aware linear classifier for binary classification is trained with SGD-type algorithms, the algorithmic estimate $\hat{\theta}$ of the classifier parameter is asymptotically normal. The asymptotic covariance depends on the Hessian of the loss at the global optima, fairness constraints, and the data distribution. Note that $\phidi(\hat{\theta})$, and $\phidm(\hat{\theta})$ both depend on the data distribution as well as $\rho(\hat{\theta})$. Among these two sources of uncertainty, only $\rho(\hat{\theta})$ depends on the stochastic optimization algorithm. So, for a practitioner, it is important to understand the uncertainty arising due to the choice of the algorithm. 
\begin{assumption}\label{as:strongcon}
The loss function $L(\cdot;W):\Theta\to\mathbb{R}^+\cup \{0\}$ is a non-decreasing $\mu$-strongly convex function for all $W\in\mathbb{R}^{d+1}$.
\vspace{-0.15in}
\end{assumption}
\begin{assumption}\label{as:lipgrad}
The loss function $L(\theta;W)$ has Lipschitz continuous gradient, i.e., for all $\theta_1,\theta_2\in\Theta$ we have,
\begin{align*}
    \norm{\nabla l(\theta_1)-\nabla l(\theta_2)}_2\leq L_G\norm{\theta_1-\theta_2}_2.
\end{align*}
where $L_G>0$ is a constant. There exists a $L_H,\delta>0$ such that, for $\theta\in \Theta\cap\{\theta:\norm{\theta-\theta^*}_2\leq \delta\}$, we have
\begin{align*}
    \norm{\nabla l(\theta)-\nabla l(\theta^*)-\nabla^2 l(\theta)(\theta-\theta^*)}_2\leq L_H\norm{\theta-\theta^*}_2^2.
\end{align*}
\end{assumption}
\begin{assumption}\label{as:noisevar}
There is a constant $L_E$, such that, for all $\theta\in\Theta^*$, and $W\sim\mathcal{D}$,
\begin{align*}
    \expec{\norm{\nabla L(\theta,W)-\nabla L(\theta^*,W)}_2^2}{}\leq L_E\norm{\theta-\theta^*}_2^2.
\end{align*}
\end{assumption}
\begin{assumption}
    The data sequence $\{(W_k,z_k)\}_k$ is $\iid$, and $\expec{L(\theta,W|\theta)}{}=l(\theta)$, i.e., the noise in the gradient is unbiased. We also assume the following on the growth of the function gradient,
    \begin{align*}
        \norm{l(\theta)}_2^2\leq C(1+\norm{\theta}_2^2),
    \end{align*}
    for some constant $C>0$.
\end{assumption}
\begin{assumption}\label{as:asymvar}
    \begin{align*}
        \underset{k\to\infty}{\lim}Cov\left[\nabla L(\theta^*,W_k)\right]=\Sigma.\numberthis\label{eq:asymvar}
    \end{align*}
\end{assumption}
\begin{assumption}\label{as:lincon}
For any point $x\in\mathbb{R}^d$, the decision boundary $\delta_\theta(x)$ is a linear function of $\theta$, i.e., for some feature map $\varphi(x)$ we can write, $\delta_\theta(x)=\theta^\top\varphi(x)$.
\end{assumption}
Assumption~\ref{as:strongcon}, and Assumption~\ref{as:lipgrad} hold for classical loss functions used in classification including logistic loss and squared loss. Assumption~\ref{as:noisevar} implies that when $\theta$ is close to the optimal value $\theta^*$, the variance of the noise reduces proportionally to $\norm{\theta-\theta^*}_2^2$ \cite{polyak1992acceleration,duchi2016asymptotic,fang2018online}. It is easy to see that this holds true for logistic regression and Support Vector Machine (SVM). Assumption~\ref{as:lincon} holds true for logistic regression, SVM as well as in the Neural Tangent Kernel regime of deep neural network classifier. 
\vspace{-0.07in}
\subsection{Disparate Impact}\label{sec:di}
We adopt the DI-aware classification framework similar to the one presented in \cite{zafar2019fairness} given by,
\begin{align*}
    \min& \quad l(\theta)\coloneqq\expec{L(\theta;W)}{}\\
    \text{subject to}&\quad \abs{\expec{  (z-\expec{z}{})\delta_{\theta}(x)}{}}\leq \epsilon. \numberthis\label{eq:mainprob}
\end{align*}
where $\delta_\theta(x)$ is the decision boundary, and $\epsilon>0$ is the unfairness tolerance level. Our setting differs from \cite{zafar2019fairness} in the following two aspects:
\begin{enumerate*}
    \item We assume an online setting, i.e., the training data is not available at once but arrive sequentially one sample at a time. This is especially efficient when the training dataset is large and computing the gradient of the loss function on the whole dataset is expensive.
    \item We assume access to a fixed small dataset $\mathcal{D}'$ with $n_c=|\mathcal{D}'|$ samples $\iid$ as the training data. We use this dataset to impose the fairness constraints on the classification problem. 
\end{enumerate*} 
Along with Assumption~\ref{as:lincon} the above setup implies that the problem \eqref{eq:mainprob} now becomes,
\vspace{-0.05in}
\begin{align*}
    \min& \quad l(\theta)\\
    \text{subject to}&\quad \theta\in\Theta\coloneqq\left\lbrace \theta|\abs{  \tilde{x}^\top\theta}\leq \epsilon\right\rbrace, \numberthis\label{eq:mainprobsample}
\end{align*}
where $\tilde{x}=n_c^{-1}\sum_{i=1}^{n_c}(z_i-\bar{z})x_i$, and $\bar{z}$ is the sample average of $z_i\in\mathcal{D}'$.
We use the Stochastic Dual Averaging (SDA) algorithm ( Algorithm~\ref{alg:DA}) to solve \eqref{eq:mainprobsample}. In this algorithm one maintains a sequence of updates $\{z_k\}_k$ which is a weighted sum of all the noisy gradients up to time $k$. $\theta_k$ is set equal to the projection $\Pi_\Theta(-z_{k-1})$ of $-z_{k-1}$ on to the set $\Theta$. Since $\Theta$ is a space confined by two hyperplanes $\Pi_\Theta(\cdot)$ can be written in closed form given by,
\begin{align*}
    \Pi_\Theta(\theta)=\begin{cases}
    \theta, & \text{if } \abs{\tilde{x}^\top\theta} \leq \epsilon \\
    \theta-\frac{\tilde{x}^\top \theta-\epsilon}{\|\tilde{x}\|_2^2}\tilde{x}, & \text{if } \tilde{x}^\top\theta \geq \epsilon\\
    \theta-\frac{\tilde{x}^\top \theta+\epsilon}{\|\tilde{x}\|_2^2}\tilde{x}, & \text{if } \tilde{x}^\top\theta \leq -\epsilon
  \end{cases}
\end{align*}
\vspace{-0.05in}
\begin{algorithm}[t]
	\caption{Weighted Stochastic Dual Averaging} \label{alg:DA}
	\textbf{Input:} $\eta_k$, $z_0\in \mathbb{R}^d$
	\begin{algorithmic}[1]
		\STATE \textbf{for} $k=1,\cdots,T$ \textbf{do}
		\STATE \textbf{Update} $\theta_{k}=\argmin_{\theta\in\Theta}\left\lbrace\left\langle z_{k-1},\theta\right\rangle+\frac{1}{2}\norm{\theta}_2^2\right\rbrace=\Pi_{\Theta}(-z_{k-1})$ 
		\STATE $z_k=z_{k-1}+\eta_k\nabla L(\theta_k,W_k)$ 
		\STATE \textbf{end for}
	\end{algorithmic}	
	\text{where $\Pi_\Theta$ is the projection operator on to the set $\Theta$.}\\
 \textbf{Output:} $\bar{\theta}_T=\frac1T\sum_{k=1}^T\theta_k\in \mathbb{R}^d$.
\end{algorithm}
\vspace{-0.1in}
\begin{theorem}\label{th:duchiclt}
Let Assumption~\ref{as:strongcon}-\ref{as:lincon} be true and $\eta_k=k^{-a}$ where $a\in (1/2,1)$. Then, for the updates of Algorithm~\ref{alg:DA}, we have, $\bar{\theta}_n\overset{a.s.}{\to}\theta_{DI}^*$, and 
\begin{align*}
    {\sqrt{n}}(\bar{\theta}_n-\theta_{DI}^*)\sim^dN\left(0,\Sigma_{DI}\right),
\end{align*}
where $\theta_{DI}^*$ is the optima of \eqref{eq:mainprobsample}, $P_A=I-\tilde{x}\tilde{x}^\top/\norm{\tilde{x}}_2^2$, $\Sigma_{DI}=P_A\left(\nabla^2 l(\theta^*)\right)^\dagger P_A\Sigma P_A\left(\nabla^2 l(\theta^*)\right)^\dagger P_A$. 
\end{theorem}
\vspace{-0.05in}
The proof of Theorem~\ref{th:duchiclt} is provided in Appendix~\ref{pf:duchiclt}.
\begin{figure*}[t] 
    \centering
    \subfigure[]{\label{fig:W2_logistic_synthetic}\includegraphics[width=40mm,height=35.5mm]{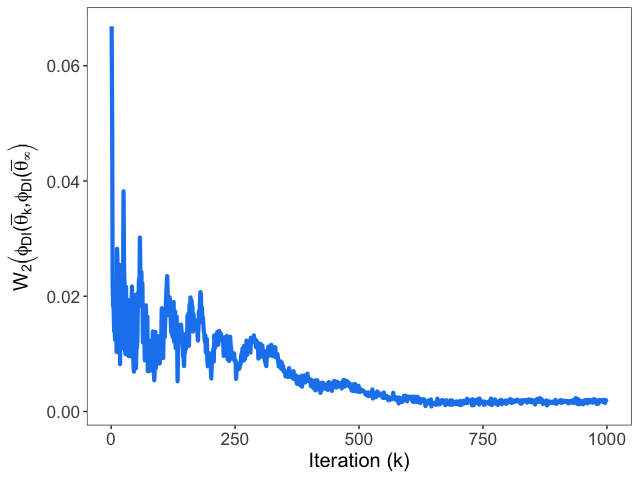}}
    \subfigure[]{\label{fig:CI_logistic_synthetic}\includegraphics[width=40mm,height=35.5mm]{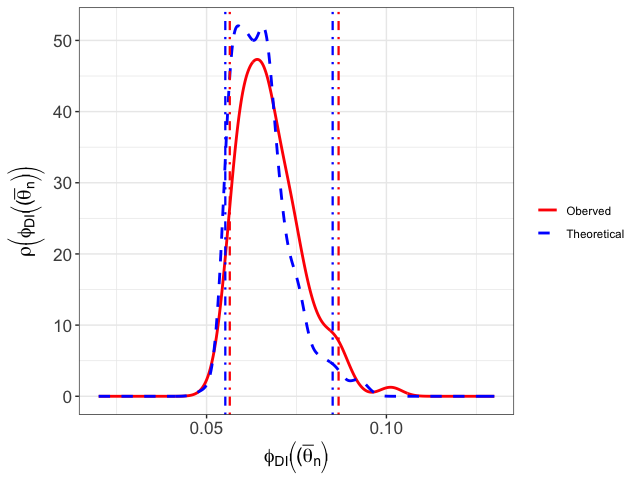}}
    \subfigure[]{\label{fig:boot_coverage_logistic_synth}\includegraphics[width=40mm]{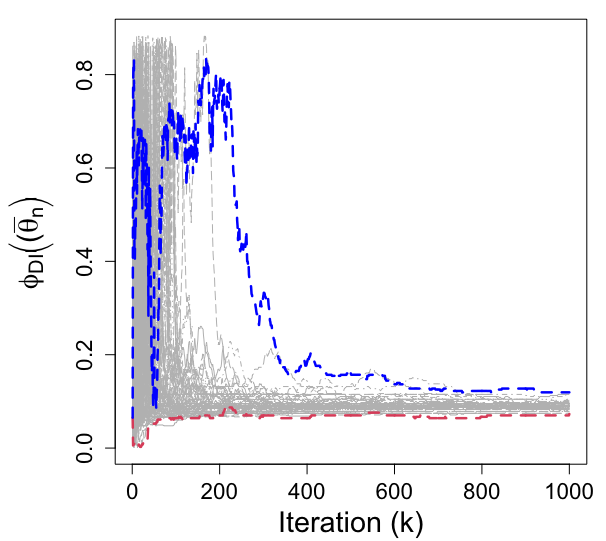}}
    \subfigure[]{\label{fig:MIS_logistic_synth}\includegraphics[width=40mm]{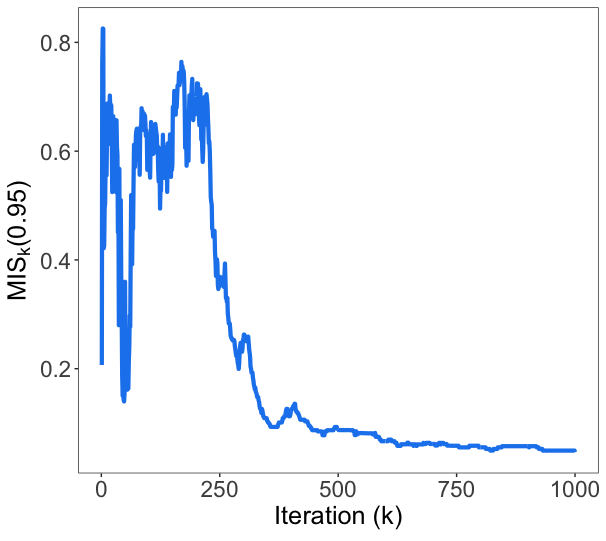}}
    \caption{(a) Convergence of $W_2\left(\phi_{DI}(\bar{\theta}_k),\phi_{DI}(\bar{\theta}_\infty)\right)$ (b) comparison of theoretical asymptotic density and observed density of $\phi_{DI}(\bar{\theta}_k)$ (c) Online Bootstrap 95\% CI (d) $MIS(\hat\phi_{DI,0.025},\hat\phi_{DI,0.975};0.95)$ under CE loss for synthetic data at risk of DI.}\label{fig:logistic_synth_DI}
    \end{figure*}
\subsection{Disparate Mistreatment}\label{sec:DM}
The constrained formulation, similar to \eqref{eq:mainprobsample}, of linear classifier designed to counter unfairness due to DM with respect to FPR, is as follows \cite{zafar2019fairness},
\begin{align*}
    &\min_{\theta\in\Theta} \quad l(\theta), \quad\text{where}\\
    &\Theta\coloneqq\left\lbrace \theta|\abs{  \frac{1}{n_c^-}\sum_{i\in {\mathcal{D}'}^-}(z_i-\bar{z})\min(0,y_ix_i^\top\theta)}\leq \epsilon\right\rbrace, \numberthis\label{eq:mainprobsampledm}  
\end{align*}
and ${\mathcal{D}'}^-$ contains the samples from $\mathcal{D}'$ with label $-1$, and $n_c^-=|{\mathcal{D}'}^-|$. Note that now we have a convex-concave constraint set \cite{zafar2019fairness}. In general, finding the global minima of an optimization problem with nonconvex constraints is an algorithmically hard problem. So existing literature often concentrate on finding approximate minima \cite{eichfelder2021nonconvex}, stationary points \cite{lin2019inexact}, and local minima \cite{boob2022stochastic}. In a nonconvex landscape, presence of multiple (often infinite) stationary points and/or local minima imply that a CLT type result is not well-defined. But we can exploit the piecewise-linear nature of our decision boundary to circumvent this problem. Instead of the constrained formulation we propose the following penalty-based approach.
\begin{align*}
    \min& \quad l(\theta)+ R_2\gamma(\theta;{\mathcal{D}'}^-), \quad \text{where,} \numberthis\label{eq:mainprobsampledmuncon}  
\end{align*}
$$
\gamma(\theta;{\mathcal{D}'}^-)\coloneqq({(n_c^-)}^{-1}\sum_{i\in{\mathcal{D}'}^-}(z_i-\bar{z})\min(0,y_ix_i^\top\theta))^2.
$$
Similar approaches have been proposed in \cite{wan2021modeling,celis2019classification}. We assume that $R_2$ is user-provided similar to $\epsilon$ in DI formulation. Our results hold for any $R_2\geq 0$. The points of discontinuity of $\min(0,y_ix_i^\top\theta)$ are given by $y_ix_i^\top\theta=0$. If the distribution of $x_i$ does not contain a point mass, the hessian of $\min(0,y_ix_i^\top\theta)$ is $0$ almost everywhere. We can bypass this issue by defining the Hessian and the gradient at $y_ix_i^\top\theta=0$ to be $0$ \cite{blanc2020implicit}. Another way of bypassing this issue would be to replace the function $\min(0,y_ix_i^\top\theta)$ by a differentiable function $-\log(1+\exp(-\tau y_ix_i^\top\theta))/\tau$, $\tau>0$ which can approximate $\min(0,y_ix_i^\top\theta)$ arbitrarily close as $\tau\to\infty$. It is easy to see that the spectral norm of the Hessian of $-\log(1+\exp(-\tau y_ix_i^\top\theta))/\tau$ is proportional to $\tau\exp(-\tau y_ix_i^\top\theta)/(1+\exp(-\tau y_ix_i^\top\theta))^2\to 0$ as $\tau\to\infty$. Then the objective function in \eqref{eq:mainprobsampledm} is strongly convex. So, to solve \eqref{eq:mainprobsampledm}, we use vanilla SGD with Polyak-Ruppert averaging \cite{polyak1992acceleration} given by:
\begin{align*}
		    &\theta_{k}=\theta_{k-1}-\eta \left(\nabla L(\theta_{k-1},w_k) +R_2\nabla\gamma(\theta_{k-1};{\mathcal{D}'}^-)\right) \\
      &\bar{\theta}_T=\frac1T\sum_{k=1}^T\theta_k\numberthis\label{eq:sgd}
\end{align*} 
Polyak-Ruppert averaging provably achieves optimal asymptotic variance under our assumptions \cite{ruppert1988efficient,polyak1992acceleration,tripuraneni2018averaging}. Then we have the following result on asymptotic distribution of $\bar{\theta}_T$. 
\begin{theorem}\label{th:dmclt}
    Let Assumptions~\ref{as:strongcon}-\ref{as:lincon} be true, and $H_{DM}\coloneqq\nabla^2 l(\theta_{DM}^*)$. Then, choosing $\eta_k=k^{-a}$ with $1/2<a<1$, for the updates of SGD \eqref{eq:sgd} we have, $\bar{\theta}_n\overset{a.s.}{\to}\theta_{DM}^*$, and 
    \begin{align*}
        \sqrt{k}\left(\bar{\theta}_k-\theta_{DM}^*\right)\sim N(0, \Sigma_{DM} )
    \end{align*}
    where $\theta_{DM}^*$ is the global optima of \eqref{eq:mainprobsampledm}, and $\Sigma_{DM}=H_{DM}^{-1}\left(\Sigma+R_2^2\nabla \gamma(\theta;{\mathcal{D}'^-})\nabla\gamma(\theta;{\mathcal{D}'^-})^\top\right)H_{DM}^{-1}$.
\end{theorem}
We omit the proof of this result here since it follows readily from Theorem 3 of \cite{polyak1992acceleration}.
\begin{remark}
One can not construct a CI for the algorithmic estimator $\bar{\theta}_n$ based on Theorem~\ref{th:duchiclt}
and Theorem~\ref{th:dmclt} because $\Sigma_{DI}$, and $\Sigma_{DM}$ depend on the true minimizer which is unknown. So we need to estimate this covariance to use these asymptotic results to construct CI. We address this issue in Section~\ref{sec:onlinecov} which is our main theoretical contribution. 
\end{remark}
\vspace{-0.15in}
\begin{figure*}[h] 
    \centering
    \subfigure[]{\label{fig:W2_logistic_adult}\includegraphics[width=40mm,height=35.5mm]{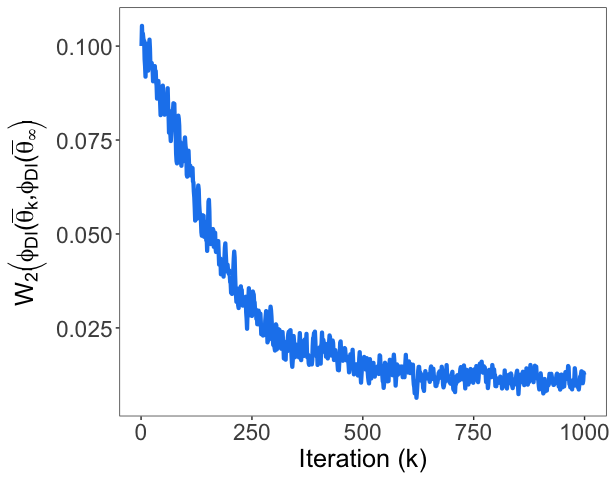}}
    \subfigure[]{\label{fig:CI_logistic_adult}\includegraphics[width=40mm,height=35.5mm]{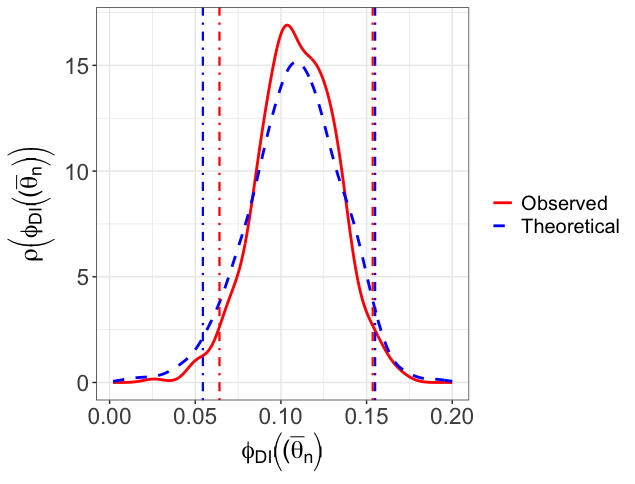}}
     \subfigure[]{\label{fig:boot_coverage_logistic_adult}\includegraphics[width=40mm]{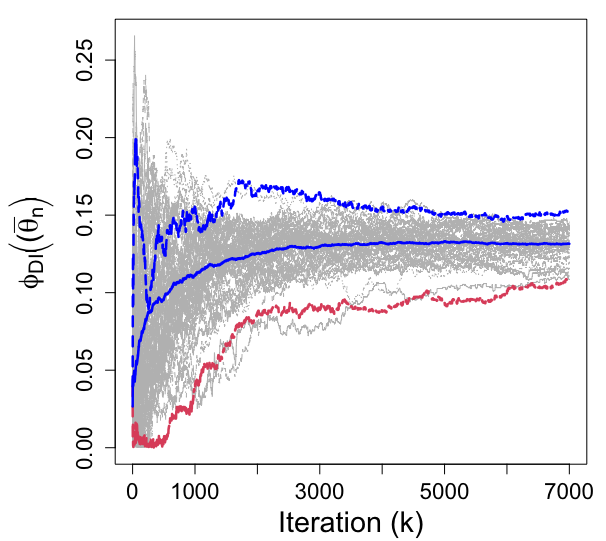}}
    \subfigure[]{\label{fig:MIS_logistic_adult}\includegraphics[width=40mm]{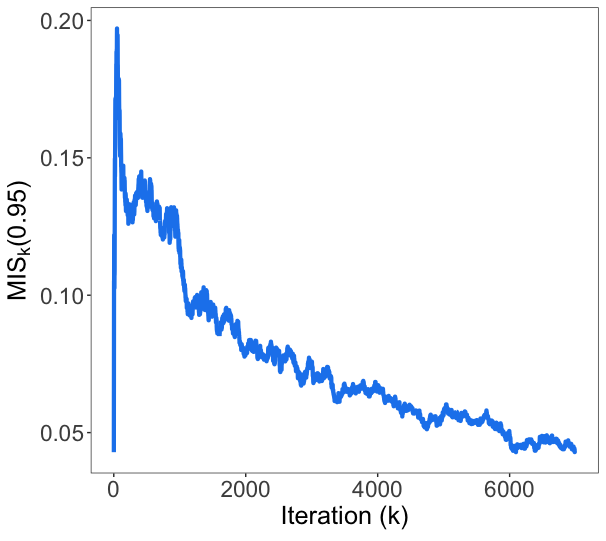}}
    \caption{(a) Convergence of $W_2\left(\phi_{DI}(\bar{\theta}_k),\phi_{DI}(\bar{\theta}_\infty)\right)$ (b) comparison of theoretical asymptotic density and observed density of $\phi_{DI}(\bar{\theta}_k)$ (c) Online Bootstrap 95\% CI (d) $MIS(\hat\phi_{DI,0.025},\hat\phi_{DI,0.975};0.95)$ under CE loss for Adult data at risk of DI.}\label{fig:logistic_adult_di}
    \end{figure*}
\section{Methodology: Online CI Estimation}\label{sec:onlinecov}
In this section we present the bootstrap-based online CI estimation algorithm. We split the estimation into two parts:
\begin{enumerate}[wide=0pt]
\vspace{-0.1 in}
    \item Firstly, we use an online bootstrap method (Algorithm~\ref{alg:boot}) to generate a sample $\{\bar{\theta}_k^b\}_{b=1}^B$ of size $B$ at each time $k$ to estimate the distribution of the averaged iterate of the optimization algorithm.
    \vspace{-0.05in}
    \item To estimate the CI of fairness using $\{\bar{\theta}_k^b\}_{b=1}^B$, we use a small held-out dataset $\tilde{\mathcal{D}}$ independent of $\mathcal{D}'$. Let $n_0$ and $n_1$ denote the number of samples in $\tilde{\mathcal{D}}$ with $z_i=0$ and $z_i=1$ respectively. Let ${n}_{0,-}$ and ${n}_{1,-}$ denote the number of samples in $\tilde{\mathcal{D}}$ with $z_i=0,y_i=-1$ and $z_i=1,y_i=-1$ respectively. Using \eqref{eq:phididef} and \eqref{eq:phidmdef}, we generate $B$ samples of fairness estimates $\{\hat{\phi}_{DI}(\bar{\theta}_k^b)\}_{b=1}^B$ and $\{\hat{\phi}_{DM}(\bar{\theta}_k^b)\}_{b=1}^B$. For any given $\theta$, $\hat{\phi}_{DI}(\theta)$, and $\hat{\phi}_{DM}(\theta)$ are evaluated on $\tilde{\mathcal{D}}$ as
    \begin{equation}\label{eq:didm}
\begin{aligned}
  &\hat{\phi}_{DI}({\theta})=\abs{n_0^{-1}\sum_{i\in\tilde{\mathcal{D}},z_i=0}u_i-n_1^{-1}\sum_{i\in\tilde{\mathcal{D}},z_i=1}u_i} \\
   & \hat{\phi}_{DM}({\theta})=\abs{\frac{\sum_{i\in\tilde{\mathcal{D}},z_i=0,y_i=-1}u_i}{{n_{0,-}}}-\frac{\sum_{i\in\tilde{\mathcal{D}},z_i=1,y_i=-1}u_i}{n_{1,-}}},
\end{aligned}
\end{equation}
where $u_i=\mathbbm{1}({{\theta}}^\top x_i>0)$.
    For $0<\alpha<1$, to estimate CI of significance level $\alpha$ of ${\phi}_{DI}(\bar{\theta}_k)$, and ${\phi}_{DM}(\bar{\theta}_k)$ we use the sample CIs $[\hat\phi_{DI,\alpha/2},\hat\phi_{DI,1-\alpha/2}]$, and $[\hat\phi_{DM,\alpha/2},\hat\phi_{DM,1-\alpha/2}]$ of $\{\hat{\phi}_{DI}(\bar{\theta}_k^b)\}_{b=1}^B$ and $\{\hat{\phi}_{DM}(\bar{\theta}_k^b)\}_{b=1}^B$ respectively. 
\end{enumerate}  
\vspace{-0.1in}
\begin{algorithm}[h]
	\caption{Online Bootstrap for $1-\alpha$ CI Estimation} \label{alg:boot}
	\textbf{Input:}$\{\eta_k\}_k$, $\theta_0,z_0\in\mathbb{R}^d$, $\mathcal{V}$, fairness criterion, $\alpha$
	\begin{algorithmic}[1]
		\STATE \textbf{for} $k=1,\cdots,T$ \textbf{do}
		\STATE \textbf{for} $b=1,2,\cdots,B$
        \IF{fairness criterion = DI}
		\STATE \textbf{Update} $\theta_{k}^b=\Pi_{\Theta}(-z_{k-1}^b)$ 
		\STATE \textbf{Sample} $V_{k}^b\sim\mathcal{V}$
		\STATE $z_{k}^b=z_{k-1}^b+\eta_{k}V_{k}^b\nabla L(\theta_{k}^b,W_{k})$ 
        \STATE $\bar{\theta}_k^b=k^{-1}(\bar{\theta}_{k-1}^b+\theta_k^b)$
        \ENDIF
        \IF{fairness criterion = DM }
        \STATE $\theta_{k}^b=\theta_{k-1}^b-\eta_{k}V_{k}^b\nabla L(\theta_{k}^b,W_{k})$
        \STATE $\bar{\theta}_k^b=k^{-1}(\bar{\theta}_{k-1}^b+\theta_k^b)$
        \ENDIF
		\STATE \textbf{end for}
		\STATE \textbf{end for}
  \STATE Compute $\{\hat{\phi}_{DI[DM]}(\bar{\theta}_{T}^b)\}_{b=1}^B$ using \eqref{eq:didm}
    \STATE \textbf{Output:}
    \STATE $[\hat\phi_{DI[DM],\alpha/2},\hat\phi_{DI[DM],1-\alpha/2}]$
	\end{algorithmic}	
\end{algorithm}
It is easy to generate $\{\bar{\theta}_k^b\}_{b=1}^B$ when one has access to $B$ independent data streams. Then one can simply run the algorithm $B$ times on independent data sequences to generate $B$ samples of $\bar{\theta}_k$. But we have to estimate the distribution of $\bar{\theta}_k$ in an online fashion with access to only one sequence of data-points. So we use online bootstrap method (Algorithm~\ref{alg:boot}). \cite{fang2018online} shows that the online bootstrap algorithm can approximate the asymptotic distribution of SGD-based estimators for strongly-convex objectives in an unconstrained setting. So these results in Theorem 2 and Theorem 3 of \cite{fang2018online} are readily applicable for SGD \eqref{eq:sgd} but not for Algorithm~\ref{alg:DA}. In this work, we extend these results to the constrained setting in order to provide similar guarantees for Algorithm~\ref{alg:DA}. To the best of our knowledge, this is the first theoretical result on online inference for constrained stochastic optimization which could be of independent interest. 

In the online bootstrap algorithm (Algorithm~\ref{alg:boot}), at every iteration $k$, we generate $B$ perturbed iterates given by line 4-6 for DI, and line 11 for DM, 
where $\{V_{k}^b\}_{k,b}$ is a sequence of $\iid$ univariate random variables from the distribution $\mathcal{V}$ with $\expec{V_{k}^b}{}=1$, $var\left[V_{k}^b\right]=1$, and $\theta_{0}^b=\theta_0$ for all $k\geq 1$, $b=1,2,\cdots,B$. Then the empirical distribution of $\bar{\theta}_n^b-\bar{\theta}_n$ is the approximation of the distribution of $\bar{\theta}_n-\theta^*$ where $\bar{\theta}_n^b=n^{-1}\sum_{i=1}^n\theta_i^b$. Observe that, in keeping with online optimization regime, generation of $\{\bar{\theta}_k^b\}_b$ require only one data point $W_k$ at iteration $k$. 
\begin{figure*}[h] 
    \centering
    \subfigure[]{\label{fig:W2_logistic_synthetic_DM}\includegraphics[width=40mm,height=35.5mm]{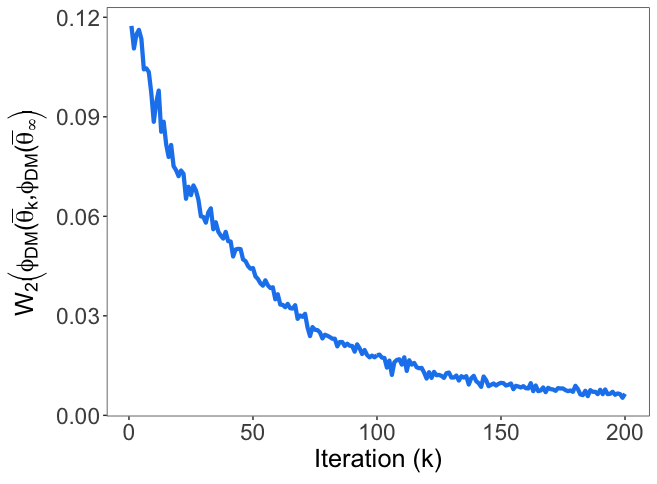}}
    \subfigure[]{\label{fig:CI_logistic_synthetic_DM}\includegraphics[width=40mm,height=35.5mm]{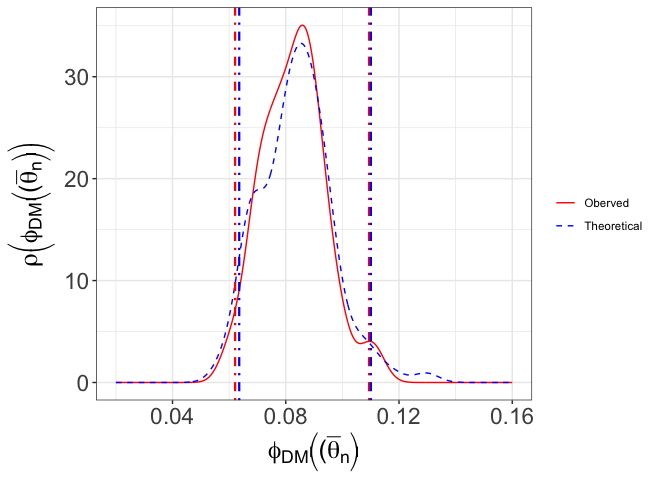}}
\subfigure{\label{fig:boot_coverage_logistic_synth_DM}\includegraphics[width=40mm]{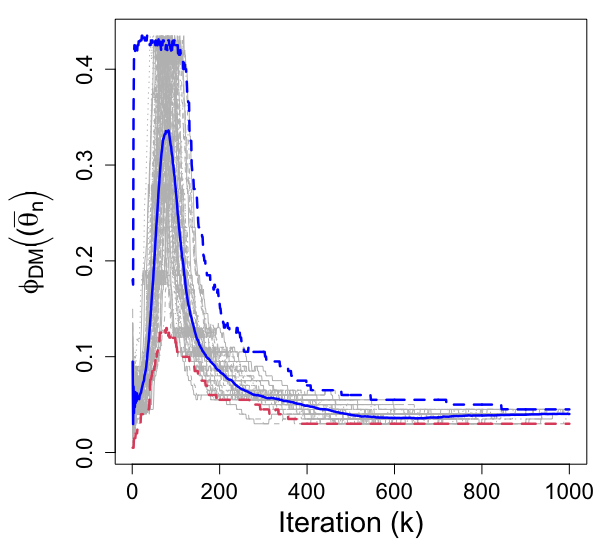}}
    \subfigure{\label{fig:MIS_logistic_synth_DM}\includegraphics[width=40mm]{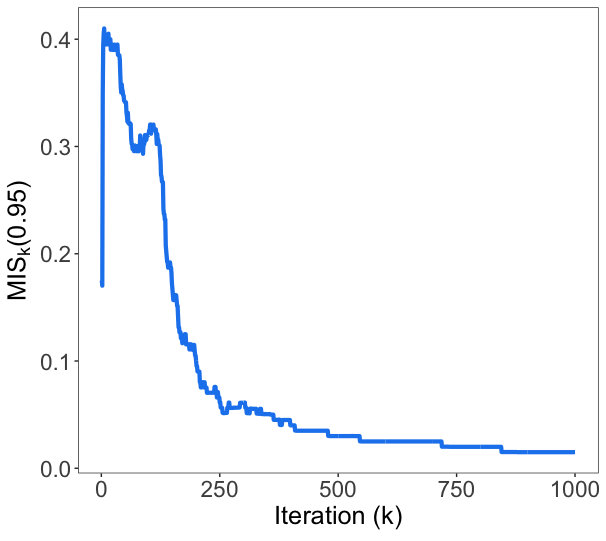}}
    \caption{(a) Convergence of $W_2\left(\phi_{DM}(\bar{\theta}_k),\phi_{DM}(\bar{\theta}_\infty)\right)$ (b) comparison of theoretical asymptotic density and observed density of $\phi_{DM}(\bar{\theta}_k)$ (c) Online Bootstrap 95\% CI (d) $MIS(\hat\phi_{DM,0.025},\hat\phi_{DM,0.975};0.95)$ under CE loss for synthetic data vulnerable to DM.}\label{fig:logistic_synthetic_DM}
    \end{figure*}
\vspace{-0.08in}
\begin{theorem}\label{th:bootdistconv}
    Let Assumptions~\ref{as:strongcon}-\ref{as:lincon} be true. Then, choosing $\eta_k= k^{-a}$, $1/2<a<1$ in Algorithm~\ref{alg:boot}, for any $b\in{1,2,\cdots,B}$, we have, $\bar{\theta}_k\overset{a.s.}{\to}\theta_{DI[DM]}^*$, and 
    \vspace{-0.1in}
    \begin{enumerate}[wide=0pt]
        \item $
    \sqrt{n}\left(\bar{\theta}_k^b-\bar{\theta}_k\right)\overset{d}{\sim}N\left(0,\Sigma_{DI[DM]}\right),
$
\item \begin{align*}
    \sup_{q\in\mathbb{R}^d}\abs{P(\sqrt{k}(\bar{\theta}_k^b-\bar{\theta}_k)\leq q)-P(\sqrt{k}(\bar{\theta}_k-\theta_{DI[DM]}^*)\leq q)}
    \overset{P}{\to} 0, 
\end{align*}
\item \begin{align*}
    &\sup_{c\in[0,1]}\left| P\left(\sqrt{k}\left(\hat{\phi}_{DI[DM]}(\bar{\theta}^b_k)-
    \phi_{DI[DM]}(\bar{\theta}_k)\right)\leq c\right)-\right.\\
    &\left. P\left(\sqrt{k}\left(\phi_{DI[DM]}(\bar{\theta}_k)-\phi_{DI[DM]}(\theta_{DI[DM]}^*)\right)\leq c\right)\right|\overset{P}{\to}0.
\end{align*}
    \end{enumerate}
\end{theorem}
\vspace{-0.1in}
\begin{remark}[On computation]
    The bootstrap procedure at most changes the gradient computation by $O(1)$ since $B$ is a constant and does not grow with iteration $n$. These bootstrap sample computations can be parallelized easily. Additionally, Algorithm~\ref{alg:boot} respects the online learning regime, and does not impose any extra data-burden since all the bootstrap samples are evaluated using the same data-point.
\end{remark}
\begin{remark}[On interpretability]\label{rm:featureimp}
     One can also construct CI for $\bar{\theta}_{k}$ based on $B$ bootstrap samples $\{\bar{\theta}_k^b\}_{b=1}^B$. This could shed light on the features which are significant for fairness-aware classification. This leads to better model interpretability.  
\end{remark}
Since \eqref{eq:mainprobsampledm} is an unconstrained strongly convex problem and we use SGD to optimize, the proof of the theorem for $\phi_{DM}(\bar{\theta}_k^b)$ readily follows from Theorem 2 of \cite{fang2018online}. We only prove Theorem~\ref{th:bootdistconv} for $\phi_{DI}(\bar{\theta}_k^b)$. This is our main theoretical contribution. We provide a proof outline here while deferring the detailed proof to the Appendix~\ref{pf:bootdistconv}.

\textbf{Proof Outline:}
First we show since we are multiplying by $\iid$ perturbation variable $\mathcal{V}_k^b$ whose mean 1, we do not lose the unbiasedness of the gradient noise. By Assumption~\ref{as:noisevar}, and the fact that $V_k^b$ is independent of $\theta_k^b$, and $W_k$, we have the following regularity condition on the perturbed noisy gradient analogous to Assumption~\ref{as:noisevar},
\begin{align*}
\expec{{V_k^b}^2\norm{\nabla L(\theta,W)-\nabla L(\theta^*,W)}_2^2}{}\leq 2L_E\norm{\theta-\theta^*}_2^2. 
\end{align*} 
Then using Theorem 3 of \cite{duchi2016asymptotic} we show that the bootstrapped iterates identify the active constraints of the form $A^b\theta_k^b=\epsilon$ correctly after finite number of steps $K$. This implies that for $k\geq K$, $P_A\theta_k=\theta_k$ where,
$P_A=I-{A}^\top(A{A}^\top)^\dagger  A=I-\tilde{x}\tilde{x}^\top/\norm{\tilde{x}}_2^2$. Then we show that the projected bootstrapped iterates $P_A\theta_k^b$ satisfies the following,
\begin{align*}
    \sqrt{k}(\bar{\theta}_k^b-\theta^*)\overset{d}{\to}N(0,\Sigma_{DI}).
\end{align*}
Then we show that,
\begin{align*}
    \sqrt{n}\left(\bar{\theta}_n^b-\bar{\theta}_n\right)=&P_H^\dagger\frac{1}{\sqrt{n}}\sum_{k=1}^n (V_k^b-1)P_A\nabla L(\theta_{DI}^*,W_k)+\\
    &o_P(1),
\end{align*}
where $P_H=P_A\nabla^2 l(\theta_{DI}^*)P_A$.
Then rest of the proofs of the first 2 parts mainly follow by using Martingale CLT theorem \cite{hall2014martingale}. The proof of the consistency of $\hat{\phi}_{DI}(\bar{\theta}_k^b)$ predominantly depends on delta method. 

We use Mean Interval Score (MIS) to evaluate the quality of the estimated CI. This is a better scoring function to evaluate interval predictions compared Brier score, and Continuous Ranked Probability Score; see \cite{wu2021quantifying} for details. Given a sample of size $N$ of a random variable $s$, and an estimate $[l,u]$ of CI of level $\alpha$, $MIS_N(l,u;\alpha)$ is defined as
\begin{align*}
   MIS_N(l,u;\alpha)&=u-l+\frac{2}{N\alpha}\sum_{i=1}^N(s_i-u)\mathbbm{1}(s_i>u)\\
   &+\frac{2}{N\alpha}\sum_{i=1}^N(l-s_i)\mathbbm{1}(s_i<l) .\numberthis\label{eq:MISdef}
\end{align*}
It is shown in \cite{wu2021quantifying} that $MIS_N(l,u;\alpha)$ is minimized at $[l^*,u^*]$ where $[l^*,u^*]$ is the true $(1-\alpha)$ CI. Qualitatively, MIS favors narrower well-calibrated interval. 
\vspace{-0.1in}
\section{Results}\label{sec:results}
In this section we present our results on synthetic and real datasets. For each dataset, we show that the observed density converges to the theoretical asymptotic density in $(W_2)$ distance, and MIS for the CI provided by Algorithm~\ref{alg:boot} decreases over time. Across all experiments, to find the global optima we initially run the algorithm once until $\norm{\theta_k-\theta_{k-1}}_2<10^{-7}$. Since there is no guideline to choose $\epsilon$, we chose $\epsilon$ for which the mean accuracy level under CE loss matches with \cite{zafar2019fairness}. We study the two most widely-used loss functions, namely Cross-Entropy (CE), and squared loss. CE loss is strictly convex and squared loss is almost surely strongly convex for $\iid$ data. We add a regularizer $R(\theta)=\varkappa\norm{\theta}_2^2/2$ to the canonical expected losses to ensure strong convexity. Note that $\varkappa$ can be tuned to ensure sparsity in high-dimensional case. We choose small $\varkappa>0$ since our goal is not sparsity but just to ensure strong convexity. The detailed form of loss function are presented in Appendix~\ref{sec:lossfn}. The plots corresponding to squared loss are in Appendix~\ref{sec:sqplots} due to lack of space. To generate bootstrap samples we use uniform random variable $V_k^b\sim\mathcal{V}=U[1-\sqrt{3},1+\sqrt{3}]$.
\begin{figure*}[h] 
    \centering
    \subfigure[]{\label{fig:W2_logistic_compas_DM}\includegraphics[width=40mm,height=35.5mm]{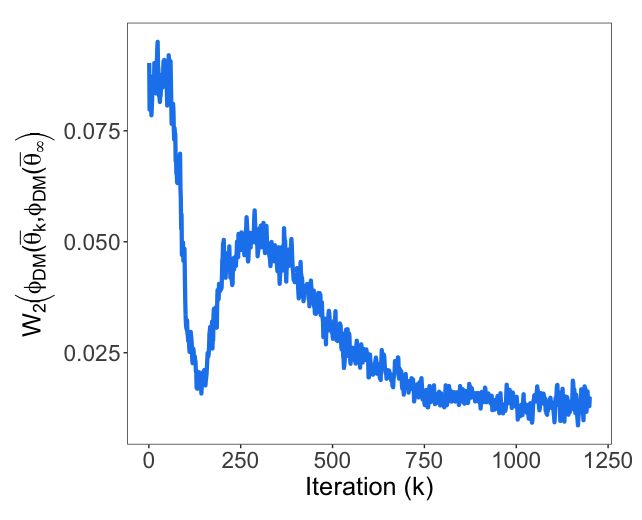}}
    \subfigure[]{\label{fig:CI_logistic_compas_DM}\includegraphics[width=40mm,height=35.5mm]{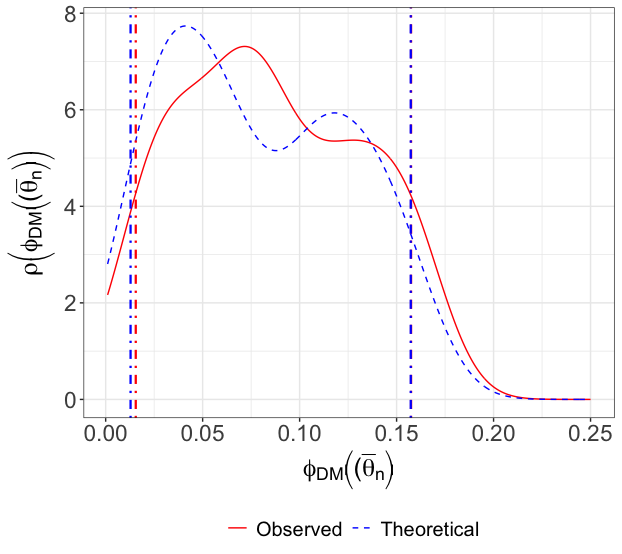}}
    \subfigure[]{\label{fig:boot_coverage_logistic_compas}\includegraphics[width=40mm,height=35.5mm]{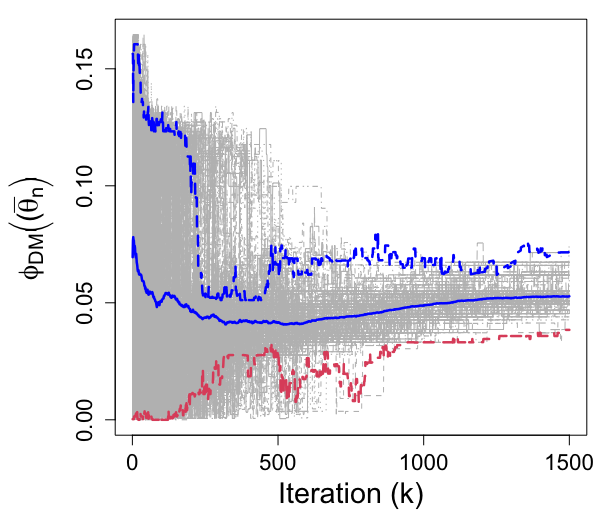}}
    \subfigure[]{\label{fig:MIS_logistic_compas}\includegraphics[width=40mm,height=35.5mm]{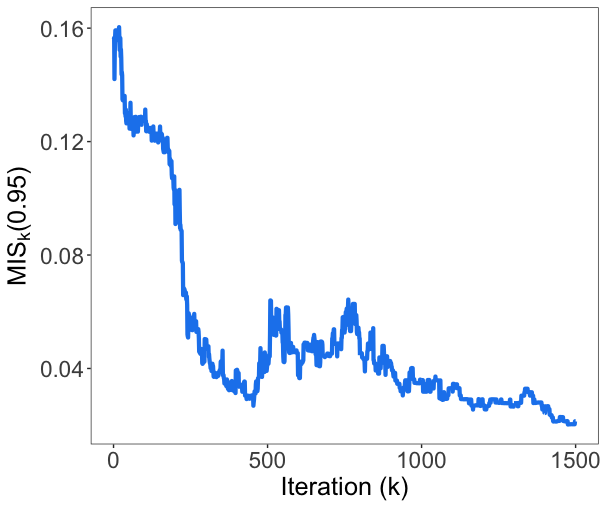}}
    \caption{(a) Convergence of $W_2\left(\phi_{DM}(\bar{\theta}_k),\phi_{DM}(\bar{\theta}_\infty)\right)$ (b) comparison of theoretical asymptotic density and observed density of $\phi_{DM}(\bar{\theta}_k)$ (c) Online Bootstrap 95\% CI (d) $MIS(\hat\phi_{DM,0.025},\hat\phi_{DM,0.975};0.95)$ under CE loss for COMPAS data.}\label{fig:logistic_compas_DM}
    \end{figure*}
\subsection{Disparate Impact}
\subsubsection{Synthetic Dataset}\label{sec:synthdi}
Similar to \cite{zafar2019fairness}, we choose $\rho(x|y=1)=N\left(\begin{bmatrix}1.5;1.5\end{bmatrix},\begin{bmatrix}0.4,0.2;0.2,0.3\end{bmatrix}\right)$, and $\rho(x|y=1)=N\left(\begin{bmatrix}-1.5;-1.5\end{bmatrix},\begin{bmatrix}0.6,0.1;0.1,0.4\end{bmatrix}\right)$. Then we choose the sensitive attribute $z$ as a Bernoulli random variable with $P(z=1)=\rho(x'|y=1)/(\rho(x'|y=1)+\rho(x'|y=-1))$, where $x'=\begin{bmatrix}\cos(\pi/3),-\sin(\pi/3);\sin(\pi/3),\cos(\pi/3)\end{bmatrix}x$.In Figure~\ref{fig:logistic_synth_DI}: (a) shows that $W_2(\phi_{DI}(\bar{\theta}_k),\phi_{DI} tv(\bar{\theta}_\infty)$, where $\bar{\theta}_\infty\sim N(\theta_{DI}^*,\Sigma_{DI}/k)$, decreases over iterations; (b) shows that after $1000$ ($4000$) iterations, the observed density of $\phi_{DI}(\bar{\theta}_{1000})$ ($\phi_{DI}(\bar{\theta}_{4000})$), and the theoretical density given by $\phi_{DI}(\bar{\theta}_\infty)$ (dashed line) almost overlap. The vertical dashed lines mark the observed (red) and theoretical (blue) $95\%$ CI (c)-(d) show that DI values across repetitions are well contained in the computed CI and $MIS(\hat\phi_{DI,0.025},\hat\phi_{DI,0.975};0.95)$ becomes small respectively. In (c), the grey lines show the variation over repetitions. 

\subsubsection{Adult Dataset}
Here we use the \texttt{Adult} dataset \cite{Dua:2019} which contains income data of adults. Two classes indicate whether income is $\geq 50K$ or $< 50K$. We use $13$ features for classification and \texttt{Sex} as the sensitive feature here. More details on the dataset is provided in the Appendix. Similar to Figure~\ref{fig:logistic_synth_DI}, in Figure~\ref{fig:logistic_adult_di}: (a)-(b) shows the convergence to the asymptotic distribution in terms of $W_2$ distance, and CI. Online CI estimation results are in (c)-(d). 
\vspace{-0.05in}
\subsection{Disparate Mistreatment}
\subsubsection{Synthetic Dataset}
We use the synthetic dataset introduced in \cite{zafar2019fairness}, i.e., we choose $\rho(x|z=0,y=1)=\rho(x|z=1,y=1)=N\left(\begin{bmatrix}2;2\end{bmatrix},\Sigma_1\right)$, $\rho(x|z=0,y=-1)=N\left(\begin{bmatrix}1;1\end{bmatrix},\Sigma_1\right)$, and $\rho(x|z=1,y=-1)=N\left(\begin{bmatrix}-2;-2\end{bmatrix},\Sigma_1\right)$ where $\Sigma_1=\begin{bmatrix}3,1;1,3\end{bmatrix}$. In Figure~\ref{fig:logistic_synthetic_DM}: (a) shows that $W_2(\phi_{DM}(\bar{\theta}_k),\phi_{DM}(\bar{\theta}_\infty))$, where $\bar{\theta}_\infty\sim N(\theta_{DM}^*,\Sigma_{DM}/k)$, becomes small over iterations; (b) shows that the observed density of $\phi_{DM}(\bar{\theta}_{1000})$, and the theoretical density given by $\phi_{DI}(\bar{\theta}_\infty)$ almost overlap. Similar results to Section~\ref{sec:synthdi} on online bootstrap CI estimation are in (c)-(d). 

\subsubsection{COMPAS Dataset}
We use the Propublica COMPAS dataset \cite{larson2016surya} containing data about criminal defendants. The goal is to classify subjects into two classes representing whether the subject will recideviate within two years (positive class) or not (negative class).  After performing the data processing as in \cite{zafar2019fairness}, we have 5278 subjects with 5 features. We use \texttt{Race} as the sensitive feature. In Figure~\ref{fig:logistic_compas_DM}: (a)-(b) show convergence to the asymptotic distribution in terms of $W_2$ distance and CI. (c)-(d) show that the estimated CI tightly contains the trajectories over multiple repetitions, and $MIS(\hat\phi_{DM,0.025},\hat\phi_{DM,0.975};0.95)$ captures the CI width correctly.
\vspace{-0.1in}
\section{Discussion and Future Work}\label{sec:disc}
In this work we show that asymptotic normality holds for the model parameter of a linear binary classifier when trained with SDA and SGD under DI and DM constraints respectively. Since the asymptotic covariances $\Sigma_{DI}$, and $\Sigma_{DM}$ depend on the unknown global minimizers $\theta_{DI}^*$, $\theta_{DM}^*$, we propose an online bootstrap-based CI estimation method for $\phi_{DI}(\bar{\theta}_n)$, and $\phi_{DM}(\bar{\theta}_n)$. To this end, we extend the theoretical results on online covariance estimation for unconstrained stochastic optimization in \cite{fang2018online} to the constrained setting. The main novelty of this work is that to the best of our knowledge, this is the first work to study UQ of DI and DM as well as establish theoretical online CI estimation results for constrained stochastic optimization. We illustrate our results on synthetic and real datasets. Various future research directions can be explored. Extending this work to nonlinear, especially neural network classifiers, is an important and quite challenging problem. Inference of fairness for data with more than one sensitive features is another intriguing direction.
\clearpage
\bibliography{fairness}
\bibliographystyle{alpha}
\newpage
\appendix
\onecolumn
\section{Appendix}
\subsection{Proof Theorem~\ref{th:duchiclt}}
\begin{proof}[Proof of Theorem~\ref{th:duchiclt}]\label{pf:duchiclt}
Here the constraint can be split into two constraints $\tilde{x}^\top \theta\leq \epsilon$ and $(-\tilde{x})^\top \theta\leq \epsilon$. Clearly only one of these constraints can be active at a time. At the optima $\theta=\theta_{DI}^*$, let's denote the active constraint by $A\theta^*_{DI}=\epsilon$. Then the KKT condition is given by,
\begin{align*}
    \nabla l(\theta^*_{DI})+\lambda^*A=0,
\end{align*}
where $\lambda^*>0$.
Since our constraint is linear, using Assumption~\ref{as:strongcon}, for any $s\in\mathbb{R}^d$, we have
\begin{align*}
    s^\top\left[\nabla^2l(\theta^*_{DI})+\nabla^2(A\theta^*_{DI})\right]s\geq \mu \norm{s}_2^2. \numberthis\label{:kktstrong}
\end{align*}
This implies Assumption C of \cite{duchi2016asymptotic}.
We first need the following result from \ref{th:duchiclt}.
\begin{lemma}[Theorem 3 \cite{duchi2016asymptotic}]\label{lm:constrid}
Let Assumption~\ref{as:strongcon}-\ref{as:noisevar}, and Assumption~\ref{as:lincon} hold. Then, with probability one, there exists some $K<\infty$, such that for $k\geq K$,
\begin{align*}
    A\theta_k=\epsilon.
\end{align*}
\end{lemma}
This implies that after a finite number of steps, Algorithm~\ref{alg:DA} almost surely identifies the active set at the optima $A$ correctly. Let $P_A$ denote the projection matrix onto the null space of the active constrained. Then, almost surely,
\begin{align*}
    P_A=I-A^\top(AA^\top)^\dagger A=I-\tilde{x}(\tilde{x}^\top\tilde{x})^\dagger \tilde{x}^\top=I-\frac{\tilde{x} \tilde{x}^\top}{\norm{\tilde{x}}_2^2}.
\end{align*}
Then Theorem~\ref{th:duchiclt} follows from Theorem 4 of \cite{duchi2016asymptotic}.
\end{proof}
\subsection{Proof of Theorem~\ref{th:bootdistconv}}
We first show that $\bt_k^b$ converges almost surely to $\theta^*$ for any $b=1,2,\cdots,B$. In Algorithm~\ref{alg:boot}, the stochastic gradient term $\nabla L(\theta_k^b,W_k)$ is replaced by the perturbed stochastic gradient $V_k^b\nabla L(\theta_k^b,W_k)$. Note that by Assumption~\ref{as:noisevar}, and the fact that $V_k^b$ is independent of $\theta_k^b$, and $W_k$, we have the following inequality for the perturbed noisy gradient,
\begin{align*}
\expec{{V_k^b}^2\norm{\nabla L(\theta,W)-\nabla L(\theta^*,W)}_2^2}{}\leq 2L_E\norm{\theta-\theta^*}_2^2.\numberthis\label{eq:bootnoisevar}
\end{align*}
Then using Theorem 2 of \cite{duchi2016asymptotic} we have, 
\begin{proposition}[Theorem 2, \cite{duchi2016asymptotic}]
Under Assumption~\ref{as:strongcon}-\ref{as:noisevar}, for any $b\in\{1,2,\cdots,B\}$, we have,
\begin{align*}
    \bar{\theta}_k^b\overset{a.s.}{\to}\theta^*.
\end{align*}
\end{proposition}
We re-state Theorem~\ref{th:bootdistconv} here for convenience. 
\begin{theorem}\label{th:bootdistconvapp}
    Let Assumptions~\ref{as:strongcon}-\ref{as:lincon} be true. Then, choosing $\eta_k= k^{-a}$, $1/2<a<1$ in Algorithm~\ref{alg:boot}, for any $b\in{1,2,\cdots,B}$, we have, $\bar{\theta}_k\overset{a.s.}{\to}\theta_{DI[DM]}^*$, and 
    \vspace{-0.1in}
    \begin{enumerate}[wide=0pt]
        \item $
    \sqrt{n}\left(\bar{\theta}_k^b-\bar{\theta}_k\right)\overset{d}{\sim}N\left(0,\Sigma_{DI[DM]}\right),
$
\item \begin{align*}
    \sup_{q\in\mathbb{R}^d}\abs{P(\sqrt{k}(\bar{\theta}_k^b-\bar{\theta}_k)\leq q)-P(\sqrt{k}(\bar{\theta}_k-\theta_{DI[DM]}^*)\leq q)}
    \overset{P}{\to} 0, 
\end{align*}
\item \begin{align*}
    \sup_{c\in[0,1]}\left| P\left(\sqrt{k}\left(\hat{\phi}_{DI[DM]}(\bar{\theta}^b_k)-
    \phi_{DI[DM]}(\bar{\theta}_k)\right)\leq c\right)-\right.
    \left. P\left(\sqrt{k}\left(\phi_{DI[DM]}(\bar{\theta}_k)-\phi_{DI[DM]}(\theta_{DI[DM]}^*)\right)\leq c\right)\right|\overset{P}{\to}0.
\end{align*}
    \end{enumerate}
\end{theorem}
\begin{proof}[Proof of Theorem~\ref{th:bootdistconv}]\label{pf:bootdistconv}
\begin{enumerate}
\item
Note that with \eqref{eq:bootnoisevar}, all the assumptions required for Lemma~\ref{lm:constrid} hold. Then, from Lemma~\ref{lm:constrid} we have that, after $k\geq K$, Algorithm~\ref{alg:boot} almost surely identifies the active constraint at optima. For a perturbed trajectory let $A^b$ denote this active constraint, i.e., $A^b\theta_k^b=\epsilon$ for $k\geq K$. Let $P_{A^b}$ be the orthogonal projector to the null space of $A^b$. Note that $P_{A^b}=P_A$ since $P_A$ is an even function of $A$. Then, for $k\geq K$, $P_A(\theta_k^b-\theta_{DI}^*)=\theta_k^b-\theta_{DI}^*$. So proving the asymptotic normality of the projected sequence $P_A(\theta_k^b-\theta_{DI}^*)$ is sufficient in this case. 
From Algorithm~\ref{alg:boot}, we have,
\begin{align*}
    z_{k}^b-z_{k-1}^b=\eta_{k}V_{k}^b\nabla L(\tb_{k},W_{k}).\numberthis\label{eq:zkinc}
\end{align*}
Here the constraint can be split into two constraints $\tilde{x}^\top \theta\leq \epsilon$ and $(-\tilde{x})^\top \theta\leq \epsilon$. Clearly only one of these constraints can be active at a time. Let's denote the active constraint by $A^b\theta_k^b=\epsilon$. Then by KKT conditions for the projection step (line 4 of Algorithm~\ref{alg:boot}, we have, for some $\lambda_k^b\geq 0$, and $\mu_k^b\geq 0$, 
\begin{align*}
    \tb_{k+1}+z_{k}^b+\lambda_{k}^bA^b-\mu_{k}^b A^b=0. \numberthis\label{eq:thetakbopt}
\end{align*}
Then combining \eqref{eq:zkinc}, and \eqref{eq:thetakbopt}, we get
\begin{align*}
    \tb_{k+1}=\tb_{k}-\eta_{k}V_{k}^b\nabla L(\tb_{k},W_{k})-A^b(\lambda_k^b-\lambda_{k-1}^b)+A^b(\mu_k^b-\mu_{k-1}^b).
\end{align*}
Multiplying both sides by the projection matrix $P_A=I-{A^b}^\top(A^b{A^b}^\top)^\dagger  A^b=I-\tilde{x}\tilde{x}^\top/\norm{\tilde{x}}_2^2$, and using $P_AA^b=0$, we have,
\begin{align*}
    &P_A(\tb_{k+1}-\theta_{DI}^*)\\
    =&P_A(\tb_{k}-\theta_{DI}^*)-\eta_{k}V_{k}^bP_A\nabla L(\tb_{k},W_{k})-P_A A^b(\lambda_k^b-\lambda_{k-1}^b)+P_A A^b(\mu_k^b-\mu_{k-1}^b)\\
    =&(I-\eta_{k}P_A\nabla^2l(\theta_{DI}^*)P_A)P_A(\tb_k-\theta_{DI}^*)+\eta_{k}P_A\nabla^2l(\theta_{DI}^*)P_A(\tb_k-\theta_{DI}^*)
    -\eta_{k}V_{k}^bP_A\nabla L(\tb_{k},W_{k})\\
    =&(I-\eta_{k}P_A\nabla^2l(\theta_{DI}^*)P_A)P_A(\tb_k-\theta_{DI}^*)-\eta_{k}(I-P_A)\nabla^2l(\theta_{DI}^*)P_A(\tb_k-\theta_{DI}^*)+\eta_{k}\nabla^2l(\theta_{DI}^*)P_A(\tb_k-\theta_{DI}^*)\\
    &+\eta_{k}P_A\nabla l(\theta_{DI}^*)-\eta_{k}P_A\nabla l(\theta_k^b)+\eta_{k}P_A\nabla l(\theta_k^b)-\eta_{k}V_{k}^bP_A\nabla L(\tb_{k},W_{k})\\
    =&(I-\eta_{k}P_A\nabla^2l(\theta_{DI}^*)P_A)P_A(\tb_k-\theta_{DI}^*)-\eta_{k}P_A\zeta_k^b+\eta_{k}P_A\xi_k^b+\eta_{k}P_AD_k^b+\epsilon_k^b,
\end{align*}
where 
\begin{align*}
    &\xi_k^b\coloneqq \nabla l(\theta_k^b)-V_{k}^b\nabla L(\tb_{k},W_{k}),\\
    &\zeta_k^b\coloneqq \nabla L(\theta_k^b)-\nabla l(\theta_{DI}^*)-\nabla^2 l(\theta_{DI}^*)(\theta_k-\theta_{DI}^*), \text{and}\\
    &\epsilon_k^b\coloneqq -\eta_k P_A\nabla^2 l(\theta_{DI}^*)(I-P_A)(\theta_k^b-\theta_{DI}^*).
\end{align*}
The last inequality follows from the fact that we have $P_A\nabla l(\theta_{DI}^*)=0$ by optimality properties. 
The following decomposition of $P_A\xi_k^b$ takes place.
\begin{align*}
    P_A\xi_k^b&=P_A\nabla l(\theta_k^b)-P_A\nabla l(\theta_{DI}^*)+V_{k}^bP_A\nabla L(\theta_{DI}^*,W_{k})-V_{k}^bP_A\nabla L(\tb_{k},W_{k})-V_{k}^bP_A\nabla L(\theta_{DI}^*,W_{k})\\
    &=\xi_k^b(0)+\xi_k^b(\theta_k^b),
\end{align*}
where,
\begin{align*}
    \xi_k^b(0)\coloneqq -V_{k}^bP_A\nabla L(\theta_{DI}^*,W_{k})\quad \xi_k^b(\theta_k^b)\coloneqq P_A\left(\nabla l(\theta_k^b)-\nabla l(\theta_{DI}^*)\right)+V_{k}^bP_A\left(\nabla L(\theta_{DI}^*,W_{k})-\nabla L(\tb_{k},W_{k})\right).
\end{align*}
Then, using Young's lemma, and Assumption~\ref{as:noisevar}, we have,
\begin{align*}
    \expec{\norm{\xi_k^b(\theta_k^b)}_2^2}{}\leq C\norm{\theta_k^b-\theta_{DI}^*}_2^2,
\end{align*}
for some constant $C>0$. 
Then using Proposition 2, specifically equation (65) of \cite{duchi2016asymptotic}, we have,
\begin{align*}
    \frac{1}{\sqrt{n}}\sum_{k=1}^n\left(\theta_k^b-\theta_{DI}^*\right)=(P_A\nabla^2 l(\theta_{DI}^*)P_A)^\dagger\frac{1}{\sqrt{n}}\sum_{k=1}^n V_k^bP_A\nabla L(\theta_{DI}^*,W_k)+o_P(1).
\end{align*}
Similarly, from Algorithm~\ref{alg:DA}, we have,
\begin{align*}
    \frac{1}{\sqrt{n}}\sum_{k=1}^n\left(\theta_k-\theta_{DI}^*\right)=(P_A\nabla^2 l(\theta_{DI}^*)P_A)^\dagger\frac{1}{\sqrt{n}}\sum_{k=1}^n P_A\nabla L(\theta_{DI}^*,W_k)+o_P(1).
\end{align*}
Combining the above two equations we get,
\begin{align*}
    \sqrt{n}\left(\bar{\theta}_n^b-\bar{\theta}_n\right)=(P_A\nabla^2 l(\theta_{DI}^*)P_A)^\dagger\frac{1}{\sqrt{n}}\sum_{k=1}^n (V_k^b-1)P_A\nabla L(\theta_{DI}^*,W_k)+o_P(1).
\end{align*}
Then by Martingale central limit theorem, we have that the asymptotic distribution of $\sqrt{n}\left(\bar{\theta}_n^b-\bar{\theta}_n\right)$ is Gaussian with mean 0 and asymptotic covariance, 
\begin{align*}
    \Sigma_{DI}&=\underset{n\to\infty}{\lim}\expec{(P_A\nabla^2 l(\theta_{DI}^*)P_A)^\dagger\frac{1}{n}\sum_{k=1}^n\sum_{i=1}^n (V_k^b-1)P_A\nabla L(\theta_{DI}^*,W_k)(V_i^b-1)\left(P_A\nabla L(\theta_{DI}^*,W_i)\right)^\top{(P_A\nabla^2 l(\theta_{DI}^*)P_A)^\dagger}^\top}{}\\
    \overset{(1)}{=}&\underset{n\to\infty}{\lim}(P_A\nabla^2 l(\theta_{DI}^*)P_A)^\dagger\frac{1}{n}\sum_{k=1}^n \expec{(V_k^b-1)^2}{}P_A\expec{\nabla L(\theta_{DI}^*,W_k)\nabla L(\theta_{DI}^*,W_k)^\top}{} P_A^\top{(P_A\nabla^2 l(\theta_{DI}^*)P_A)^\dagger}^\top\\
    \overset{(2)}{=}&(P_A\nabla^2 l(\theta_{DI}^*)P_A)^\dagger P_A\Sigma P_A^\top{(P_A\nabla^2 l(\theta_{DI}^*)P_A)^\dagger}^\top\\
    \overset{(3)}{=}& P_A\nabla^2 l(\theta_{DI}^*)^\dagger P_A\Sigma P_A\nabla^2 l(\theta_{DI}^*)^\dagger P_A.
\end{align*}
Equality (1) follows from the martingale difference property of $\{(V_k^b-1)\nabla L(\theta_{DI}^*,W_k)\}_k$. Equality (2) follows from the facts that $V_k^b$ is independent of $\nabla L(\theta_{DI}^*,W_k)$, $\expec{(V_k^b-1)^2}{}=1$, and Assumption~\ref{as:asymvar}. Equation (3) follows from the fact $P_A^2=P_A$ since $P_A$ is a projection matrix. This proves first statement of Theorem~\ref{th:bootdistconv}
\item
Let $\vartheta$ denote the Gaussian distribution $N(0,\Sigma_{DI})$. Then, from the previous part we can write, 
\begin{align*}
    \sup_{q\in\mathbb{R}^d}\abs{P(\sqrt{n}\left(\bar{\theta}_n^b-\bar{\theta}_n\right)\leq q)-P(\vartheta\leq q)}\overset{P}{\to}0.\numberthis\label{eq:bootclosetoavg}
\end{align*}
Similarly, from Theorem~\ref{th:duchiclt}, we have
\begin{align*}
    \sup_{q\in\mathbb{R}^d}\abs{P(\sqrt{n}\left(\bar{\theta}_n-\theta_{DI}^*\right)\leq q)-P(\vartheta\leq q)}\overset{P}{\to}0.\numberthis\label{eq:avgclosetostar}
\end{align*}
Combining \eqref{eq:bootclosetoavg}, and \eqref{eq:avgclosetostar}, we get the second part of the Theorem~\ref{th:bootdistconvapp},
\begin{align*}
    \sup_{q\in\mathbb{R}^d}\abs{P(\sqrt{n}\left(\bar{\theta}_n^b-\bar{\theta}_n\right)\leq q)-P(\sqrt{n}\left(\bar{\theta}_n-\theta_{DI}^*\right)\leq q)}\overset{P}{\to}0.\numberthis\label{eq:bootthetaconsistency}
\end{align*}
Now note that for a given $\bar{\theta}_n^b$, as $|\tilde{D}|\to\infty$,
\begin{align*}
    \hat{\phi}_{DI}(\bar{\theta}^b_n)\overset{a.s.}{\to}{\phi}_{DI}(\bar{\theta}^b_n)\numberthis\label{eq:phihattophi}
\end{align*} 
by strong law of large numbers.
\item
Since, $\phi_{DI}(\theta)$ is differentiable w.r.t $\theta$, using delta method, from \eqref{eq:avgclosetostar} we have,
\begin{align*}
    \sqrt{n}(\phi_{DI}(\bar{\theta}_n)-\phi_{DI}(\theta_{DI}^*))\overset{d}{\to}N(0,\nabla\phi_{DI}(\theta_{DI}^*)^\top\Sigma_{DI}\nabla\phi_{DI}(\theta_{DI}^*)).\numberthis\label{eq:phiavgclostophistar}
\end{align*}
Using delta method, and from \eqref{eq:bootclosetoavg} we have,
\begin{align*}
    \sqrt{n}(\phi_{DI}(\bar{\theta}_n^b)-\phi_{DI}(\bar{\theta}_n))\overset{d}{\to}N(0,\nabla\phi_{DI}(\bar{\theta}_{n})^\top\Sigma_{DI}\nabla\phi_{DI}(\bar{\theta}_{n})).\numberthis\label{eq:phibootclostophiavg}
\end{align*}
Now, since $\bar{\theta}_n\overset{a.s.}{\to}\theta_{DI}^*$, we have,
\begin{align*}
    \nabla\phi_{DI}(\bar{\theta}_{n})^\top\Sigma_{DI}\nabla\phi_{DI}(\bar{\theta}_{n})\overset{a.s.}{\to}\nabla\phi_{DI}(\theta_{DI}^*)^\top\Sigma_{DI}\nabla\phi_{DI}(\theta_{DI}^*).\numberthis\label{eq:phiavgcovclostophistarcov}
\end{align*}
Combining, \eqref{eq:phihattophi}, \eqref{eq:phiavgclostophistar}, \eqref{eq:phibootclostophiavg}, and \eqref{eq:phiavgcovclostophistarcov}, using the same argument as part 2, we have, as $|\tilde{\mathcal{D}}|\to\infty$,
\begin{align*}
    \sup_{c\in[0,1]}\abs{P\left(\sqrt{n}\left(\hat{\phi}_{DI}(\bar{\theta}^b_n)-\phi_{DI}(\bar{\theta}_n)\right)\leq c\right)-P\left(\sqrt{n}\left(\phi_{DI}(\bar{\theta}_n)-\phi_{DI}(\theta_{DI}^*)\right)\leq c\right)}\overset{P}{\to}0.\numberthis\label{eq:phibootthetaconsistencyapp}
\end{align*}
\end{enumerate}
\end{proof}

\section{Details on the Loss Functions}\label{sec:lossfn}
\subsection{Cross-Entropy Loss}
In case of CE loss $l(\theta)$ is of the form
\begin{align*}
l(\theta)=\expec{\log\left(1+\exp\left(-yx^\top\theta\right)\right)}{}+R(\theta).
\end{align*}
It is easy to see that canonical CE loss is strictly convex and hence we need $\varkappa>0$. 
\subsection{Squared Loss}
We use the one-hot encoding of the class for the squared loss. One-hot encoding of a class $c$ is given by the vector $y_v$ where $y_v[i]=\mathbbm{1}(i=c), \ i=1,2$. In this case, the decision boundary is given by, $$\delta_\theta (x)=\begin{bmatrix}1,  -1\end{bmatrix}\Theta x,$$
where $x\in\mathbb{R}^d$, and $\Theta\in \mathbb{R}^{2\times d}$. Then the optimization problem we solve here is given by
\begin{equation}
\begin{aligned}
    \min &\quad \expec{ \norm{y_{v}-\Theta x}_2^2}{}+R(\theta)\\
    \text{subject to}&\quad \frac{1}{n_c}\abs{\sum_{i=1}^{n_c} (z_i-\bar{z})\begin{bmatrix}1 , -1\end{bmatrix}\Theta x_i}\leq \epsilon. \label{eq:mainprobsquare}
    \end{aligned}
\end{equation}
Let $\Theta_i, i=1,2$ denote the $i$-th row of $\Theta$. Let $\Theta_v=\begin{bmatrix}
    \Theta_1^\top;\Theta_2^\top
\end{bmatrix}$ be the vectorized version of $\Theta$, and let $x_{v,i}=(x_i^\top, -x_i^\top)^\top$. Then the optimization problem \eqref{eq:mainprobsquare} can also written as 
\begin{align*}
     \min & \quad\expec{(y_v[1]-\Theta_1 x)^2+(y_v[2]-\Theta_2 x)^2}{}+R(\theta)\\
    \text{subject to}&\quad \frac{1}{n_c}\abs{\sum_{i=1}^{n_c}  (z_i-\bar{z})x_{v,i}^\top\Theta_v }\leq \epsilon. \numberthis\label{eq:mainprobsquare}
\end{align*}
Note that for squared loss, if the features are linearly independent, then the covariance matrix $\expec{xx^\top}{}$ is positive definite and $\varkappa$ can be set to $0$. 
\section{Further Details on Experiments}
For all the experiments we choose the step size $\eta_k=0.01k^{-0.501}$, and $\varkappa=0.0001$. The detailed choice of experimental setup is provided in Table~\ref{tab:exp}.

\begin{table}[h]
\centering
\begin{tabular}{|l|l|l|l|l|l|l|}
\hline
Dataset        & $\epsilon$              & B   & $|\tilde{\mathcal{D}}|$ & $|{\mathcal{D}}'|$& $R_2$                   & \#Repetitions \\ \hline
Synthetic (DI) & 0.002                   & 100 & 100 & 200             & \multicolumn{1}{c|}{--} & 100           \\ \hline
Adult          & 0.00001                 & 200 & 1000  & 1000           & \multicolumn{1}{c|}{--} & 200           \\ \hline
Synthetic (DM) & \multicolumn{1}{c|}{--} & 100 & 100 & 250             & 500                     & 100           \\ \hline
COMPAS         & \multicolumn{1}{c|}{--} & 200 & 400  & 250            & 300                     & 200           \\ \hline
\end{tabular}
\caption{The choice of parameters across experiments.}\label{tab:exp}
\end{table}

\newpage
\section{Plots corresponding to Squared Loss}\label{sec:sqplots}
\begin{figure}[h] 
    \centering
    \subfigure{\label{fig:W2_sqr_synthetic}\includegraphics[width=40mm]{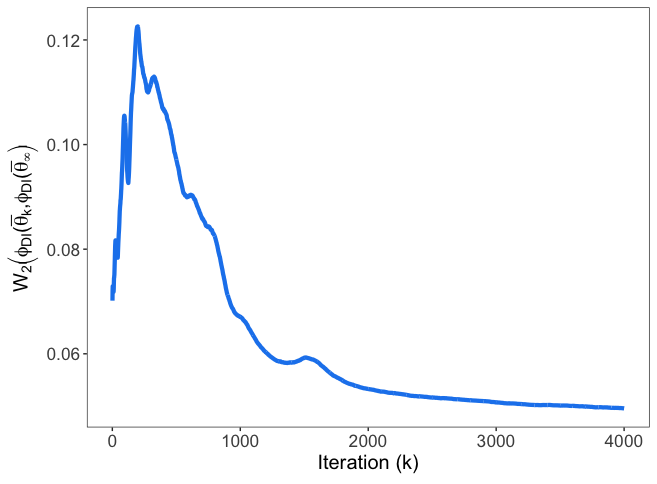}}
    \subfigure{\label{fig:CI_sqr_synthetic}\includegraphics[width=40mm]{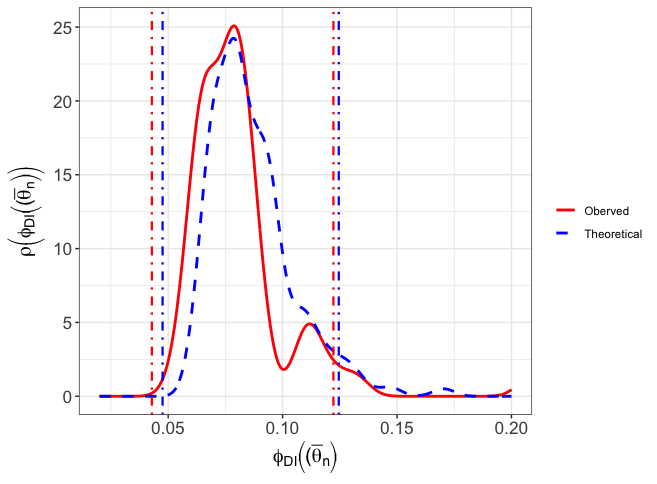}}
    \subfigure{\label{fig:boot_coverage_sqr_synth}\includegraphics[width=40mm]{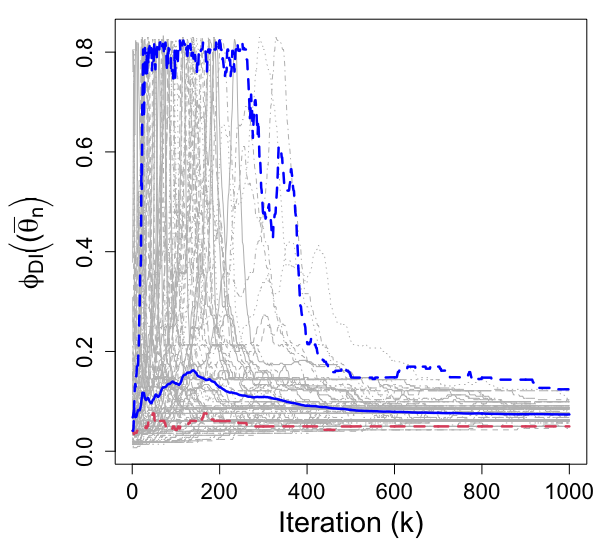}}
    \subfigure{\label{fig:MIS_sqr_synth}\includegraphics[width=40mm]{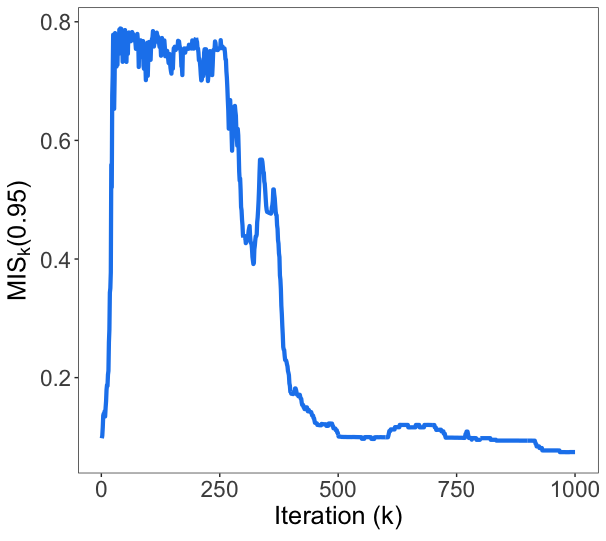}}
    \caption{Top row: Convergence of $W_2\left(\phi_{DI}(\bar{\theta}_k),\phi_{DI}(\bar{\theta}_\infty)\right)$, and comparison of theoretical asymptotic density and observed density of $\phi_{DI}(\bar{\theta}_k)$; Bottom row: Online Bootstrap 95\% CI and $MIS(\hat\phi_{DI,0.025},\hat\phi_{DI,0.975};0.95)$ under squared loss for synthetic data at risk of DI.}\label{fig:square_synth_DI}
    \vspace{-0.05in}
\end{figure}
\begin{figure}[h]
    \centering
    \subfigure{\label{fig:W2_sqr_adult}\includegraphics[width=40mm]{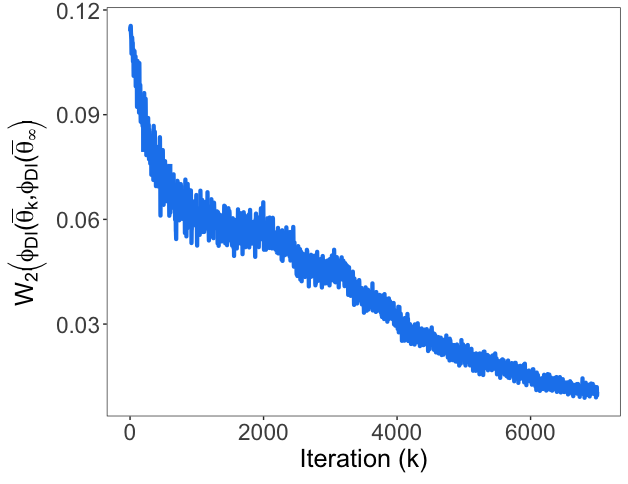}} 
    \subfigure{\label{fig:CI_sqr_adult}\includegraphics[width=40mm]{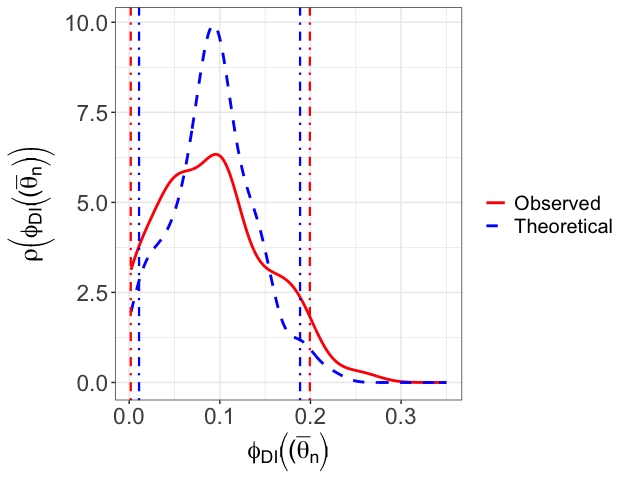}}
    \subfigure{\label{fig:boot_coverage_sqr_adult_new}\includegraphics[width=40mm]{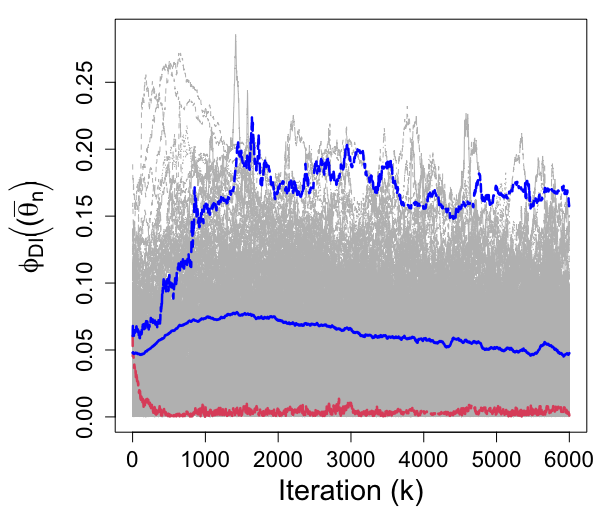}}
    \subfigure{\label{fig:MIS_sqr_adult_new}\includegraphics[width=40mm]{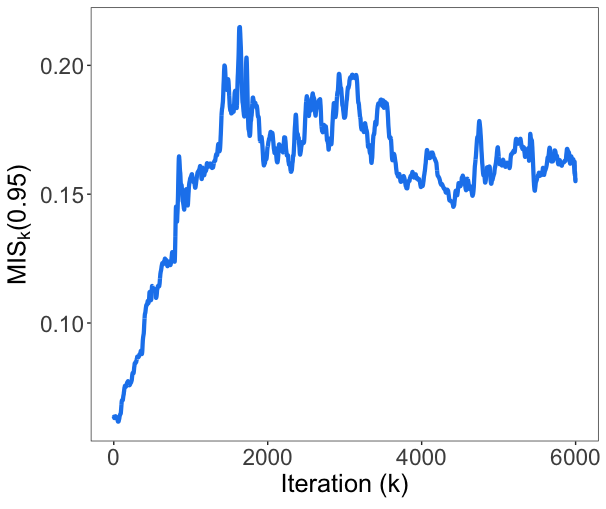}}
    \caption{Top row: Convergence of $W_2\left(\phi_{DI}(\bar{\theta}_k),\phi_{DI}(\bar{\theta}_\infty)\right)$, and comparison of theoretical asymptotic density and observed density of $\phi_{DI}(\bar{\theta}_k)$; Bottom row: Online Bootstrap 95\% CI and $MIS(\hat\phi_{DI,0.025},\hat\phi_{DI,0.975};0.95)$ under squared loss for Adult data at risk of DI.}\label{fig:square_adult_di}
    \vspace{-0.05in}
\end{figure}
 \begin{figure}
    \centering
    \subfigure{\label{fig:W2_sqr_synthetic_DM}\includegraphics[width=40mm]{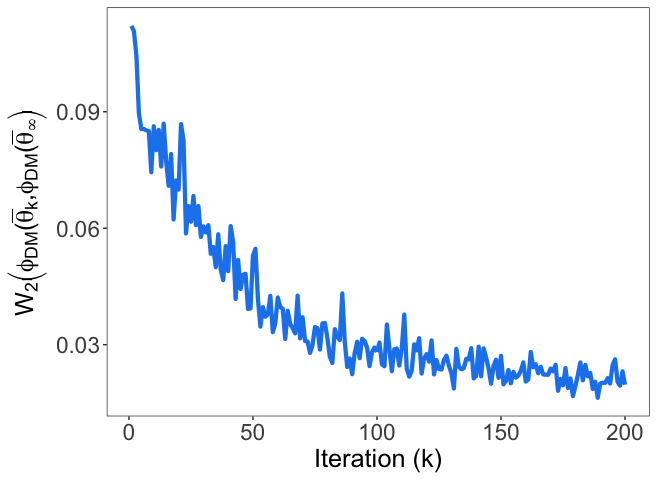}} 
    \subfigure{\label{fig:CI_sqr_synthetic_DM}\includegraphics[width=40mm]{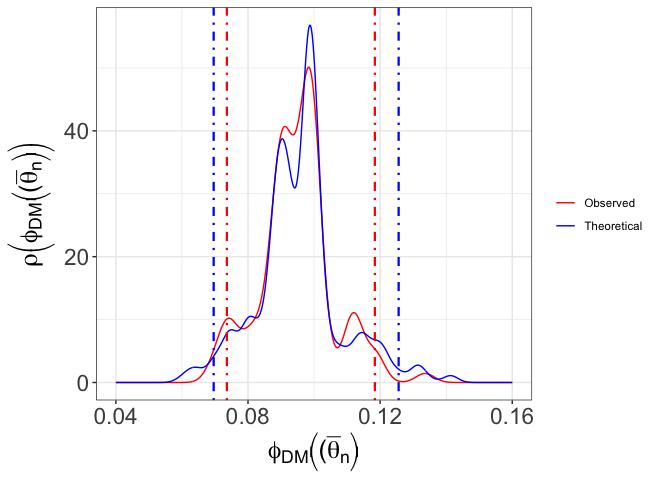}}
    \subfigure{\label{fig:boot_coverage_sqr_synth_DM}\includegraphics[width=40mm]{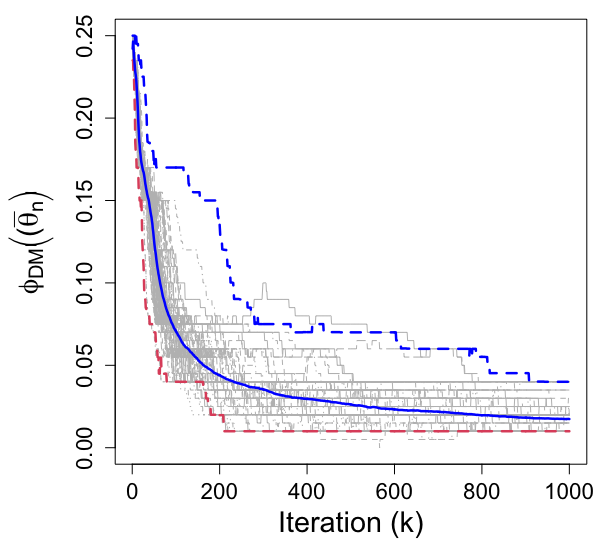}}
    \subfigure{\label{fig:MIS_sqr_synth_DM}\includegraphics[width=40mm]{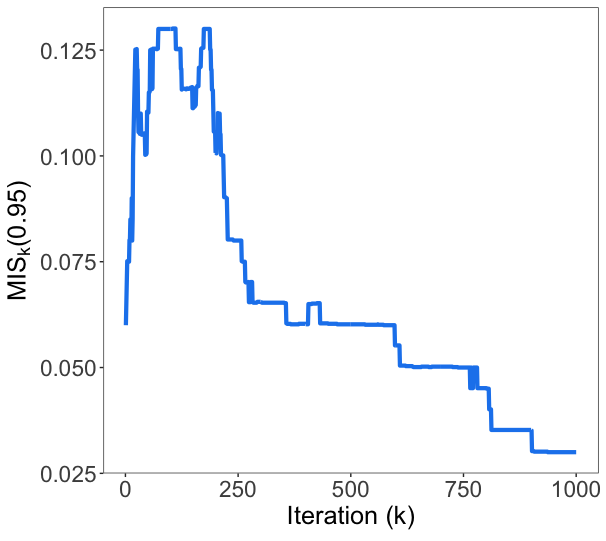}}
    \caption{Top row: Convergence of $W_2\left(\phi_{DM}(\bar{\theta}_k),\phi_{DM}(\bar{\theta}_\infty)\right)$, and comparison of theoretical asymptotic density and observed density of $\phi_{DM}(\bar{\theta}_k)$; Bottom row: Online Bootstrap 95\% CI and $MIS(\hat\phi_{DM,0.025},\hat\phi_{DM,0.975};0.95)$ under squared loss for synthetic data vulnerable to DM.}\label{fig:square_synthetic_DM}
    \vspace{-0.05in}
\end{figure}
 \begin{figure}[h]
 \centering
    \subfigure{\label{fig:W2_sqr_compas_DM}\includegraphics[width=40mm]{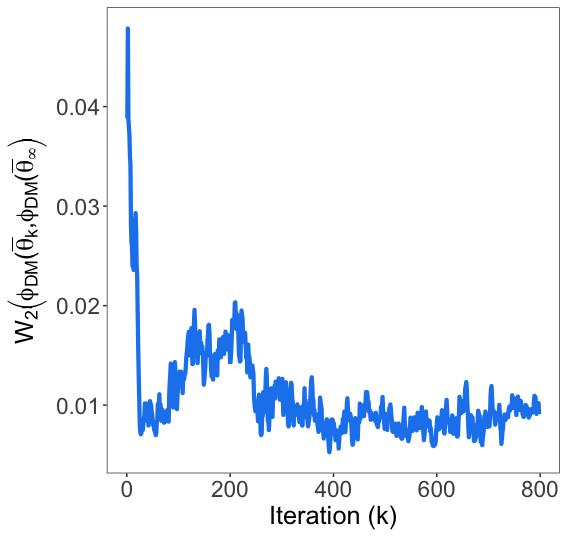}}
    \subfigure{\label{fig:CI_sqr_compas_DM}\includegraphics[width=40mm]{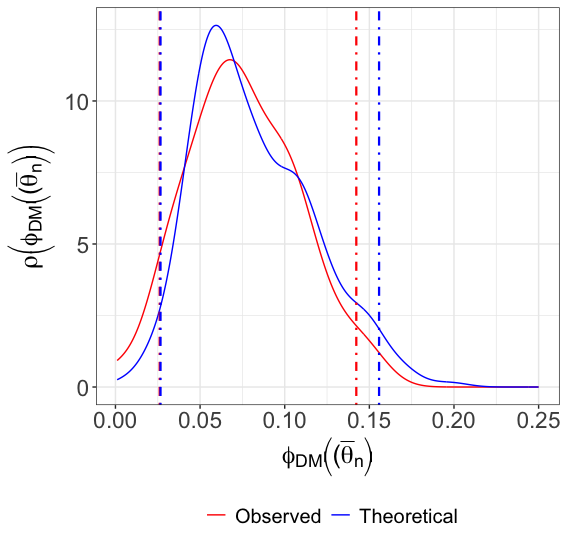}}
    \subfigure{\label{fig:boot_coverage_sqr_compas}\includegraphics[width=40mm]{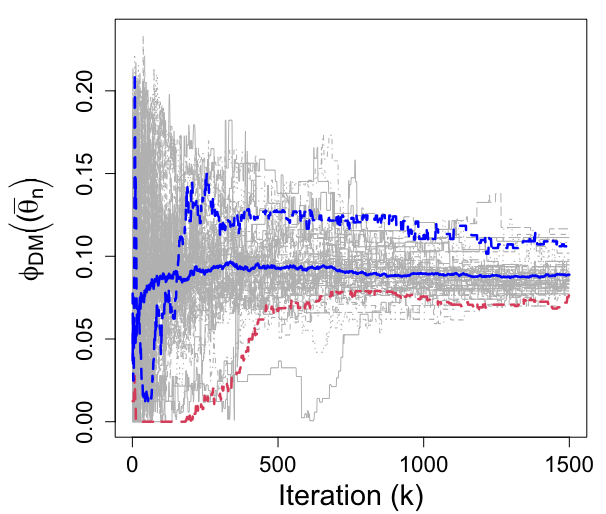}}
    \subfigure{\label{fig:MIS_sqr_compas}\includegraphics[width=40mm]{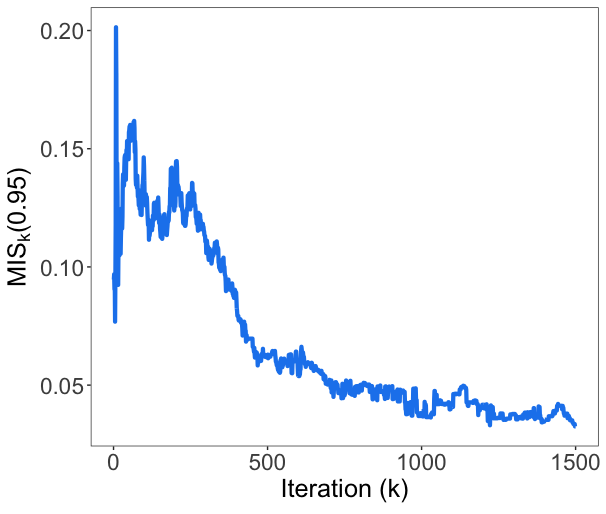}}
    \caption{Top row: Convergence of $W_2\left(\phi_{DM}(\bar{\theta}_k),\phi_{DM}(\bar{\theta}_\infty)\right)$, and comparison of theoretical asymptotic density and observed density of $\phi_{DM}(\bar{\theta}_k)$; Bottom row: Online Bootstrap 95\% CI and $MIS(\hat\phi_{DM,0.025},\hat\phi_{DM,0.975};0.95)$ under squared loss for COMPAS data.}\label{fig:square_compas_DM}
    \vspace{-0.05in}
\end{figure}
\end{document}